\documentclass[final,leqno]{siamltex}
\usepackage{moreverb}

\usepackage{graphicx,subfigure,wrapfig}
\usepackage{algorithm}
\usepackage{algpseudocode}
\usepackage{color}
\usepackage{amsmath}
\usepackage{amssymb}
\usepackage{txfonts}
\usepackage{afterpage}
\usepackage[margin=1in]{geometry}
\usepackage{mathrsfs}
\usepackage{bm}
\usepackage[american]{circuitikz}
\usepackage{mathrsfs}
\usepackage{pstricks}
\usepackage{makecell}
\usepackage{threeparttable}
\usepackage{booktabs}
\newtheorem{mlemma}{Lemma}
\newtheorem{mdefinition}{Definition}
\newtheorem{mtheorem}{Theorem}
\newtheorem{mproposition}{Proposition}
\newtheorem{massumption}{Assumption}
\newtheorem{mremark}{Remark}

\def\lt{\color{black}}
\def\ls{\color{black}}

\renewcommand{\mathbf}{\boldsymbol}

\newcommand{\mb}{\mathbf}

\newcommand{\norm}[2]{\left\| #1 \right\|_{#2}}

\title{D3M: A deep domain decomposition method for partial differential equations}

\author{Ke Li$^{\displaystyle 1}$ \thanks{School of Information Science and Technology, ShanghaiTech University, Shanghai, 200120, China. $^{\displaystyle 1}$ Equal contributions. 
		({\tt \{like1, tangkj, liaoqf\}@shanghaitech.edu.cn}).}
	\and Kejun Tang$^{\displaystyle 1}$ \footnotemark[1] 
	\and Tianfan Wu\thanks{Viterbi School of Engineering, University of Southern California, Los Angeles, USA,
		({\tt tianfanw@usc.edu}).}  
	\and Qifeng Liao\footnotemark[1] \thanks{Corresponding author.}	
}

\begin{document}

\maketitle

\begin{abstract}
A state-of-the-art deep domain decomposition method (D3M) based on the variational principle is proposed for partial differential equations (PDEs). The solution of PDEs can be formulated as the solution of a constrained optimization problem, and we design a multi-fidelity neural network framework to solve this optimization problem. Our contribution is to develop a systematical computational procedure for the underlying problem in parallel with domain decomposition. 
Our analysis shows that the D3M approximation solution converges to the exact solution of underlying PDEs. Our proposed framework establishes a foundation to use variational deep learning  in large-scale engineering problems and designs. 
We present a general mathematical framework of D3M, validate its accuracy and demonstrate its efficiency with numerical experiments.
\end{abstract}



\pagestyle{myheadings}
\thispagestyle{plain}
\markboth{K. LI, K. TANG, T. WU AND Q. LIAO}{DEEP DOMAIN DECOMPOSITION}

\section{Introduction}
\label{Introduction}
Partial differential equations (PDEs) are among the most ubiquitous tools employed 
in describing computational science and engineering problems. 
When modeling complex problems, the governing PDEs are typically expensive to solve
through transitional numerical methods, e.g., the finite element methods \cite{elman2014finite}. 
While principal component analysis \cite{wold1987principal,scholkopf1997kernel} (PCA), proper orthogonal decomposition \cite{berkooz1993proper,willcox2002balanced} (POD) and reduced basis methods \cite{verpat02,quaroz07,elman2013reduced,cheroz14,liao2016reduced,chenjiang16} are classical approaches for model reduction to
reduce the computational costs, deep learning \cite{lecun2015deep} 
currently gains a lot of interests for efficiently solving PDEs. 
There are mathematical guarantees called universal approximation theorems \cite{hornik1991approximation} stating that a single layer neural network can approximate most functions in Soblev spaces. 
Although there is still a lack of theoretical frameworks for explaining the effectiveness of multilayer 
neural networks, 
deep learning has become a widely used tool. 
Marvelous successful practices of deep neural networks 
encourages their applications to different areas, 
where the curse of dimensionality is a tormenting issue.
    
New approaches are actively proposed to solve PDEs based on deep learning techniques. E et al. \cite{weinan2017proposal,weinan2018deep} connect deep learning with dynamic system and propose 
a deep Ritz method (DRM) for solving PDEs via variational methods. 
Raissi et al. \cite{raissi2017physics1,raissi2017physics2,raissi2019physics} develop physics-informed neural networks which combine observed data with PDE models. 
By leveraging a prior knowledge that the underlying PDE model obeys a specific form, they can make accurate predictions with limited data. Long et al. \cite{long2017pde} present a feed-forward neural network, called PDE-Net, to accomplish two tasks at the same time: predicting time-dependent behavior of an unknown PDE-governed dynamic system, and revealing the PDE model that generates observed data. Later, 
Sirignano et al. \cite{sirignano2018dgm} propose  a deep Galerkin method (DGM), which is a meshfree deep learning algorithm to solve PDEs without requiring observed data~(solution samples of PDEs). When a steady-state high-dimensional parametric PDE system is considered, Zhu et al. \cite{zhu2018bayesian,zhu2019physics} propose Bayesian deep convolutional encoder-decoder networks for problems 
with high-dimensional random inputs. 
    
When considering computational problems arising in practical engineering, e.g.\ aeronautics and astronautics, 
systems are typically designed by multiple groups along disciplinary. The complexity of solving large-scale problems may take an expensive cost of hardware.
The balance of accuracy and generalization is also hard to trade off. For this reason, decomposing a given system into component parts to manage the complexity is a strategy,
and  the domain decomposition method is a traditional numerical method to achieve this goal. 
Schwarz \cite{schwarz1870ueber} proposes an iterative method for solving harmonic functions. Then this method is improved by S.L.Sobolev \cite{sobolev1936schwarz}, S.G.Michlin \cite{mikhlin1951schwarz}, P.L.Lions 
et al.\ \cite{lions1988schwarz,lions1989schwarz}. 
Domain decomposition is also employed for optimal design or control \cite{antil2010domain}, for decomposing a complex design task (e.g., decomposition approaches to multidisciplinary optimization \cite{liao2019parallel,kong2018efficient}), and for uncertainty analysis of models governed by PDEs \cite{liao2015domain,chen2015local}. 

In this work, we propose a variational deep learning solver based on domain decomposition methods,
which is referred to as the deep domain decomposition method (D3M) to implement parallel computations 
along physical subdomains. 
Especially, efficient treatments of complex boundary conditions are developed.
Solving PDEs using D3M has several benefits:

	\textbf{Complexity and generalization.} 
	D3M manages complexity at the local level. 
	Overfitting is a challenging problem in deep learning. 
	The risk of overfitting can be reduced by splitting the physical domain into subdomains, 
	so that each network focuses on a specific subdomain.
	
	\textbf{Mesh-free and data-free.}
	D3M constructs and trains networks under variational formulation. 
	So, it does not require given data, which can be potentially used for complex and high-dimensional problems.
	
	\textbf{Parallel computation.}
	The computational procedures of D3M are in parallel for different subdomains. 
	This feature helps D3M work efficiently on large-scale and multidisciplinary problems.
	
	In this work, D3M is developed based on iterative domain decomposition methods. 
	The development of using domain decomposition leads to an  independent  
	model-training procedure in each subdomain  
	in an ``offline'' phase, followed by assembling global solution using pre-computed local information in an ``online'' phase. Section~\ref{odd} reviews iterative overlapping domain decomposition methods. Section~\ref{d3m} presents the normal variational principle informed neural networks, and 
	 our D3M algorithms.  
	 A convergence analysis of D3M is discussed in Section~\ref{analysis}, and a summary of our full approach is presented in Section~\ref{d3m_summary}. Numerical studies are discussed in Section \ref{tests}. Finally, Section~\ref{conclusion} concludes the paper.

\section{Overlapping domain decomposition}\label{odd}
The Schwarz method \cite{schwarz1870ueber} is the most classical example of domain decomposition approach for PDEs, and it is still efficient with variant improvements \cite{li2018multilevel,zampini2017multilevel,zhang2018bi}.

Given a classical Poisson's equation
\begin{equation}\label{poi}
\left\{
\begin{aligned}
-\Delta u &= f&, \ \ &\mathrm{in} \ \Omega,\\
u &= 0&, \ \ &\mathrm{on} \ \partial\Omega.
\end{aligned}
\right.
\end{equation}

We divide $\Omega$ into two overlapping subdomains $\Omega_i, i=1,2$ (see Figure \ref{ill2}), where
\begin{equation}
\begin{aligned}
\Omega = \Omega_1\cup\Omega_2,\ \Gamma_1 := \partial\Omega_1\cap\Omega_2, \Gamma_2 := \partial\Omega_2\cap\Omega_1,\ \Omega_{1,2}:=\Omega_1\cap\Omega_2.
\end{aligned}
\end{equation}

\begin{figure}[H]
	\centering
	\subfigure[]{
		\centering
		\includegraphics[width=0.35\linewidth]{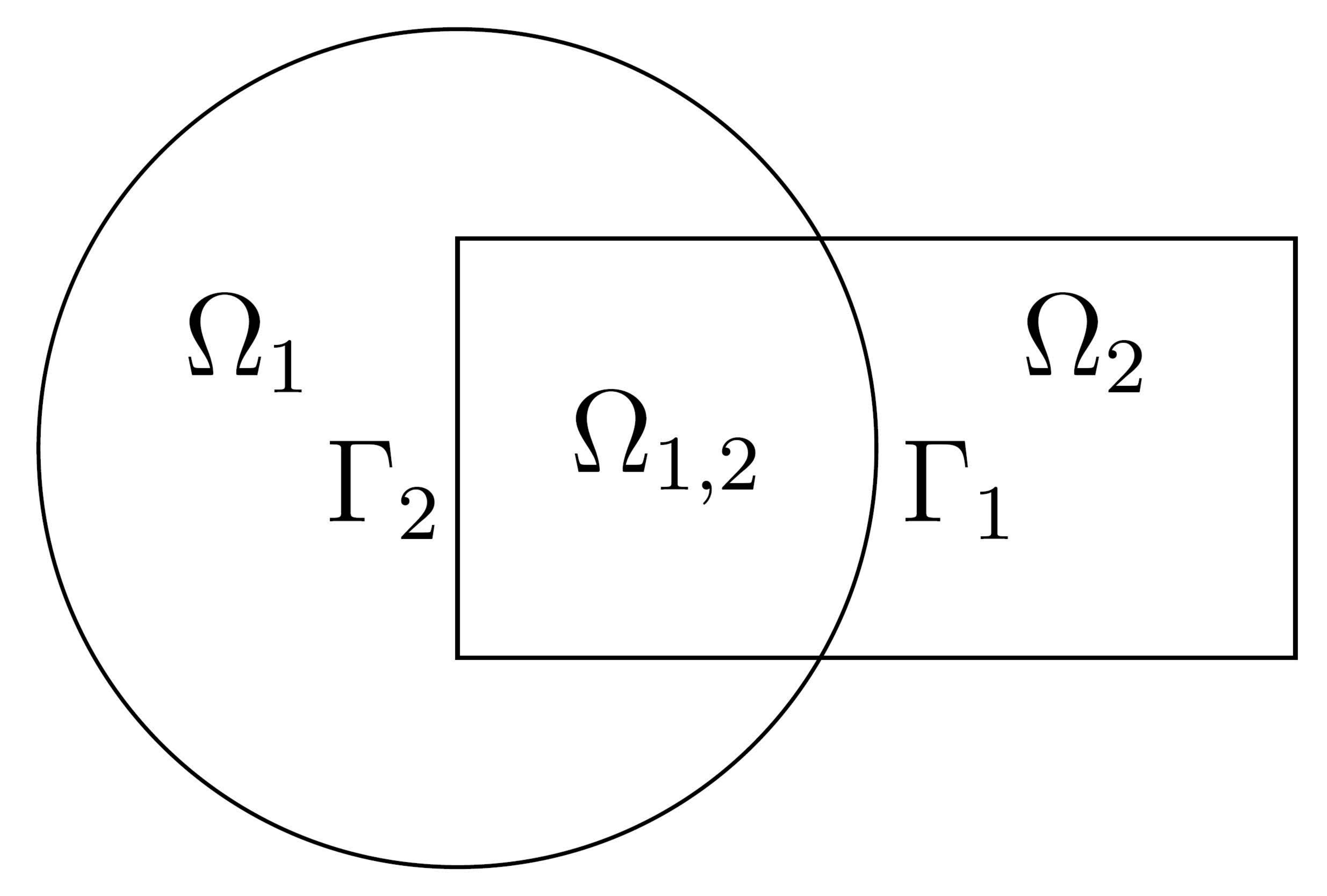}}
	\subfigure[]{
		\centering
		\includegraphics[width=0.35\linewidth]{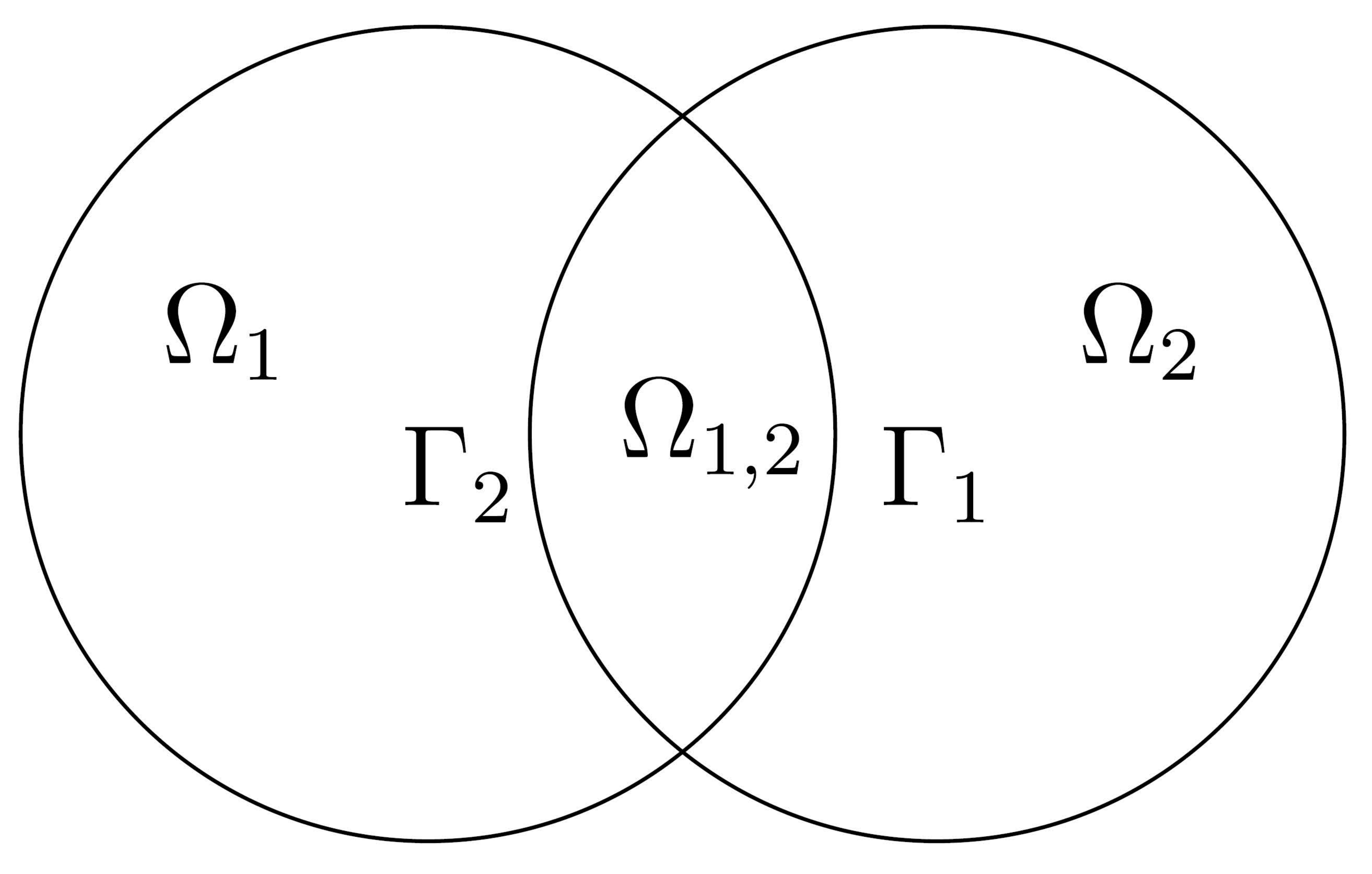}}
	\caption{Partition into two overlapping subdomains.}
	\label{ill2}
\end{figure}
We introduce the original formula named Schwarz alternating method here. Let $u^0$ be an initial guess defined in $\Omega$ and vanishing on $\partial\Omega$. For $k\geq0$, we define sequences $u_i^{k}$ where $u_i^{k}$ denotes $u^k$ in $\Omega_i$. The $u_i^{k+1}$ is determined from an iteration algorithm:
\begin{equation}\label{decom1}
\left\{
\begin{aligned}
-\Delta u_1^{k+1/2} &= f&, \ \ & \mathrm{in} \ \Omega_1,\\
u_1^{k+1/2} &= u_2^k&, \ \ & \mathrm{on} \ \Gamma_1,\\
u_1^{k+1/2} &= 0&, \ \ & \mathrm{on} \ \partial\Omega_1\cap\partial\Omega
\end{aligned}
\right.
\end{equation}
and
\begin{equation}\label{decom2}
\left\{
\begin{aligned}
-\Delta u_2^{k+1} &= f&, \ \ & \mathrm{in} \ \Omega_2,\\
u_2^{k+1} &= u_1^{k+1/2}&, \ \ & \mathrm{on} \ \Gamma_2,\\
u_2^{k+1} &= 0&, \ \ & \mathrm{on} \ \partial\Omega_2\cap\partial\Omega.
\end{aligned}
\right.
\end{equation}

\section{Deep domain decomposition method with variational principle}\label{d3m}
Before introducing D3M, we first give a brief introduction of variational principle. In this section, we consider the Poisson's equation and reformulate \eqref{poi} as a constrained minimization problem, and then we introduce the D3M algorithm. 
	
\subsection{Variational principle}
The Poisson's equation with the homogeneous Dirichlet boundary condition is \eqref{poi}, 
and we consider the situation that $f\in \mathcal{L}^2(\Omega)$ and $\Omega$ is a square domain 
in this section. The idea of the standard Deep Ritz method is based on the variational principle. 
That is, the PDE can be derived by a functional minimization problem as described in the following proposition.
\begin{mproposition}
Solving the Poisson's equation \eqref{poi}  is equivalent to an optimization problem
\begin{equation} \label{eq_variational_form}
\begin{aligned}
&\min\limits_{u} \ E(u) = \int_{\Omega} \frac{1}{2}|\nabla u|^2 dxdy - \int_{\Omega} {f \cdot u}dxdy,\\
& \ \mathrm{s.t.} \quad u = 0 \ \ \mathrm{on} \ \partial\Omega.
\end{aligned}
\end{equation} 
The Lagrangian formula of \eqref{eq_variational_form} is given by 
\begin{equation} \label{ori_L}
L(u, q) = \int_{\Omega} \frac{1}{2}|\nabla u|^2dxdy - \int_{\Omega} {u \cdot f}dxdy + q\int_{\partial\Omega} u dxdy,
\end{equation}
where $q$ is the Lagrange multiplier.
\end{mproposition}

\begin{mdefinition}
    $\mathcal{H}(\mb{div})$ denotes symmetric tensor-fields in $H^1$ space,
    in which functions are square integrable and have square integrable divergence.
\end{mdefinition}

We employ a mixed residual loss \cite{zhu2019physics} following Hellinger-Reissner principle \cite{ARNOLD1990281}.
With an additional variable $\tau\in\mathcal{H}(\mb{div})$, which represents flux, we can turn Equation \eqref{poi} into
\begin{equation}\label{mixres}
\left\{
\begin{aligned}
\tau &= - \nabla u&, \ \ &\mathrm{in} \ \Omega,\\
\nabla\cdot\tau &= f&, \ \ &\mathrm{in} \ \Omega.
\end{aligned}
\right.
\end{equation}
The mixed residual loss is 
\begin{equation} \label{g_L}
L(\tau,u,q) = \int_{\Omega}[(\tau+\nabla u)^2+(\nabla\cdot \tau-f)^2]dxdy + q\int_{\partial\Omega} u dxdy.
\end{equation}

\subsection{Variational principle informed neural networks}
Though the Poisson's equation \eqref{poi} is reformulated as an optimization problem, it is intractable to find the optimum in an infinite-dimensional function space. Instead, we seek to approximate the solution $u(x,y)$ by neural networks. We utilize $\mb{N}_u(x,y;\theta_u),\mb{N}_\tau(x,y;\theta_{\tau})$ to approximate the solution $u$ and the flux $\tau$ in domain $\Omega$, where $\theta_u$ and $\theta_{\tau}$ are the parameters to train. The input is the spatial variable in $\Omega$, and the outputs represent the function value corresponding to the input. With these settings, we can train a neural network by variational principle to represent the solution of Poisson's equation. The functional minimization problem \eqref{g_L} turns into the following optimization 
problem
\begin{equation}
\label{eq_variational_nn}
\begin{aligned}
\underset{\theta =\{\theta_u, \theta_{\tau}\}} {\mathrm{min}} \ &\int_{\Omega} \left[(\mb{N}_\tau+\nabla \mb{N}_u)^2 + (\nabla\cdot\mb{N}_\tau - f)^2\right]dxdy + q\int_{\partial\Omega}\mb{N}_u^2 dxdy.
\end{aligned}  
\end{equation}

\begin{mremark}
	In practical implementation, $\mb{N}_u$ and $\mb{N}_\tau$ are embedded in one network $\mb{N}$ parameterized with $\theta$, and the two outputs of $\mb{N}$ denote the function values of $\mb{N}_u$ and $\mb{N}_{\tau}$ respectively.
\end{mremark}

Therefore, the infinite-dimensional optimization problem \eqref{eq_variational_form} is transformed into a finite-dimensional optimization problem \eqref{eq_variational_nn}. Our goal is to find the optimal (or sub optimal) parameters $\theta$ to minimize the loss in \eqref{eq_variational_nn}. To this end, we choose a mini-batch of points randomly sampled in $\Omega$. These data points can give an estimation of the integral in \eqref{eq_variational_nn} and the gradient information to update the parameters $\theta$. 
For example, a mini-batch points $\{(x_i,y_i)\}_{i=1}^{m+n}$ are drawn in $\bar{\Omega}$ randomly, where $\{(x_i,y_i)\}_{i=1}^{m}$ in $\Omega$ and $\{(x_i,y_i)\}_{i=1}^{n}$ on $\partial\Omega$. Then the parameters can be updated by using optimization approaches
\begin{equation}
\small
	\begin{aligned}
	\theta^{(k+1)} &=  \theta^{(k)} -  
	\nabla_{\theta} \frac{1}{m}\sum\limits_{i=1}^{m} [(\mb{N}_\tau^{(i)}+\nabla \mb{N}_u^{(i)})^2 + (\nabla\cdot\mb{N}_\tau^{(i)} - f)^2]-\nabla_{\theta} \frac{1}{n}\sum\limits_{j=1}^{n}(q\cdot\mb{N}_u^{(j)})^2.
	\end{aligned}
\end{equation}

\subsection{Implementation details for neural networks}
This section provides details for the architecture of our neural networks.

For giving a direct-viewing impression, we show the implementation with a plain vanilla densely connected neural network to introduce how the structure works in Figure \ref{illu_nn}. For illustration only, the network depicted consists of 2 layers with 6 neurons in each layer. The network takes input variables $x,y$ and outputs $u,\tau=[\tau_x,\tau_y]$. The number of neurons in each layer is $M$ and $\sigma$ denotes an element-wise operator
\begin{equation}
\sigma(x)=\left(\phi\left(x_{1}\right), \phi\left(x_{2}\right), \ldots, \phi\left(x_{M}\right)\right),
\end{equation}
where $\phi$ is called the activation function. There are some commonly used activation functions such as the sigmoidal function, the tanh function, the rectified linear units (ReLU) function \cite{nair2010rectified}, and 
the Leaky ReLU function \cite{maas2013rectifier}. We employ the tanh function, and the automatic differentiation is obtained by using PyTorch \cite{paszke2017automatic}. The total loss function comprises the residual loss terms $L_1$, $L_2$ and the Lagrangian term which guarantees the constraint conditions. The parameters are trained with backpropogating gradients of the loss function and the optimizer is L-BFGS \cite{liu1989limited} where the learning rate is 0.5. In practice, the model architecture of neural networks is the residual network (ResNet) \cite{he2016deep}. These residual
networks are easier to optimize, and it can gain accuracy from considerably increased depth. The structure of ResNet improves the result of deep networks, 
{\ls because there are more previous information retained. 
A brief illustration of ResNet is in Figure \ref{resnet}.} 

\begin{figure}
	\centering
	\includegraphics[width=0.8\linewidth]{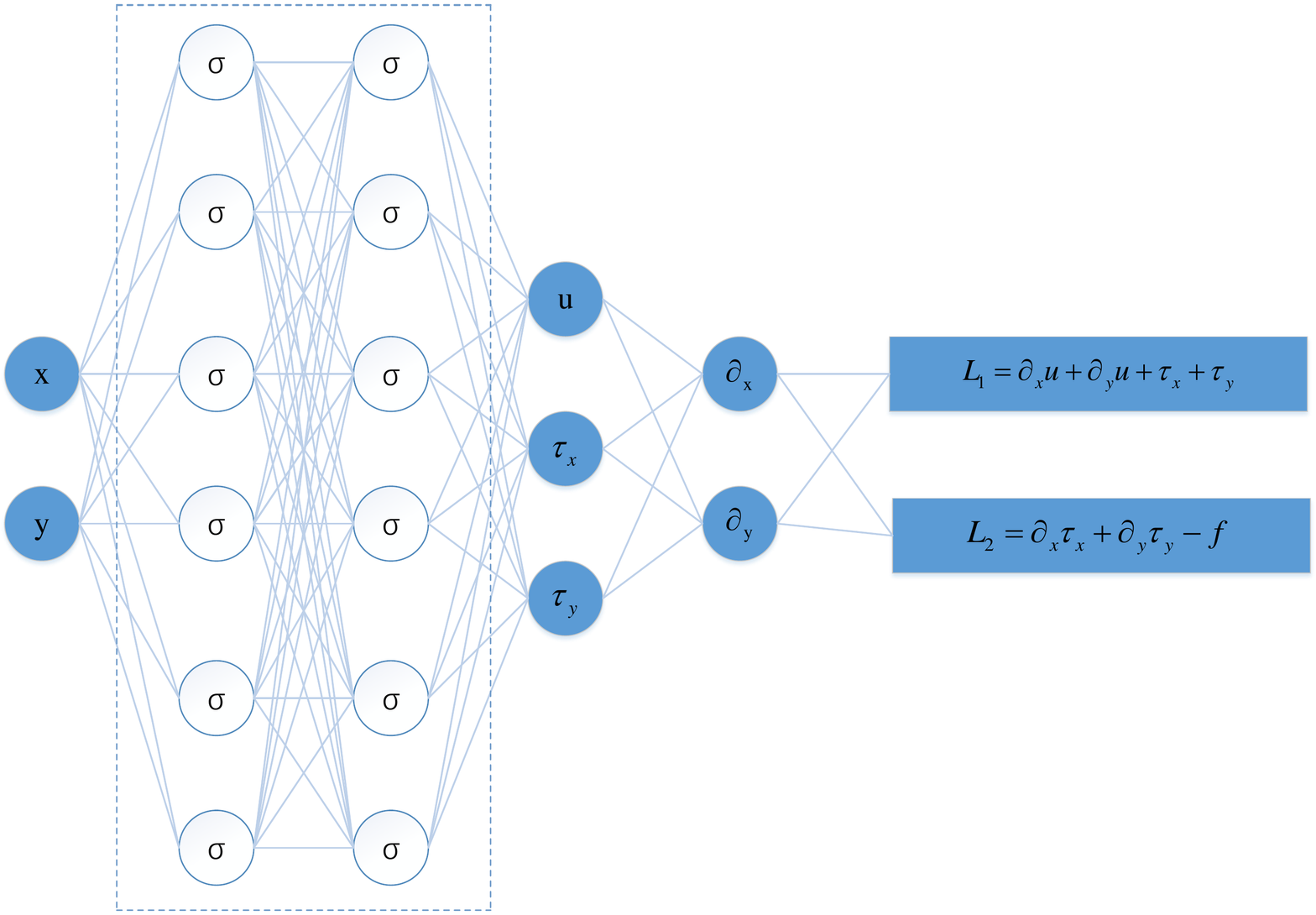}
	\caption{Illustration of the neural networks. $x,y$ are inputs, $u,\tau_x,\tau_y$ are outputs and 
	 {\ls the} dashed box with $\sigma$ means the architecture of plain fully-connected neural networks.  }
	\label{illu_nn}
\end{figure}

\begin{figure}
	\centering
	\includegraphics[width=0.3\linewidth]{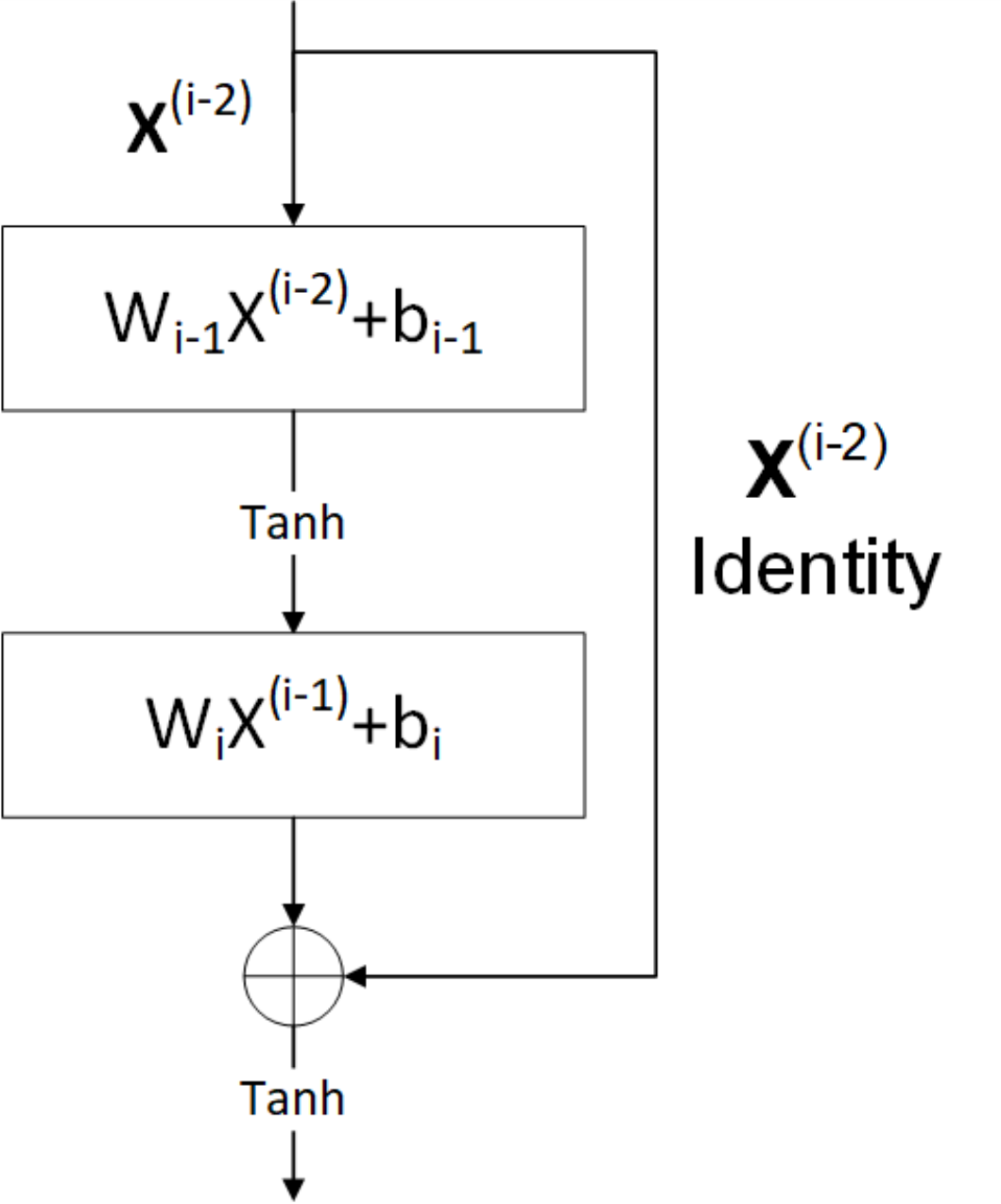}
	\caption{The residual network building block of our method.}
	\label{resnet}
\end{figure}

\subsection{Deep domain decomposition method}
In this part, we propose the main algorithms of our D3M. Because the physics-constrained neural network is mesh-free, we improve Schwarz alternating method with a better inhomogeneous D3M sampling method at junctions $\Gamma_{i}$ to accelerate convergence. 
{\ls We note the performance of normal deep variational networks and mixed residual networks 
can deteriorate when the underlying problem has inhomogeneous boundary conditions. 
Our treatment to overcome this weakness is to introduce the following boundary function.} 
\begin{mdefinition}(Boundary function)
	A smooth function $\mathfrak{g}(x,y)$ is called {\ls a} boundary function associated with $\Omega$ if
	\begin{equation}
	\mathfrak{g}(x, y) = 
	\begin{aligned}
	&e^{-a\cdot d(x,y,\partial\Omega)}u(x, y), \ \ (x, y) &\in \Omega,
	\end{aligned}
	\end{equation} 
	where $a\gg 1$ is a coefficient, the notation $d(x,y,\partial\Omega)$ denotes the shortest Euclidean distance between $(x,y)$ and $\partial\Omega$. If the point $(x,y)$ is on the boundary, $\mathfrak{g}(x,y)=u(x,y)$. If not, the value of $\mathfrak{g}(x,y)$ decreases to zero sharply. And we define $v := u - \mathfrak{g}$, 
	{\ls where} $v$ satisfies
	\begin{equation} \label{g_form}
	\left\{
	\begin{aligned}
	-\Delta v &= f + \Delta \mathfrak{g}, \ \ &\mathrm{in} \ \Omega,\\
	v &= 0, \ \ &\mathrm{on} \ \partial\Omega.
	\end{aligned}
	\right.
	\end{equation}
\end{mdefinition}
{\ls Letting} $v_i = u_i - \mathfrak{g}_i$ on each subdomain, Equation \eqref{mixres} can be represented as
\begin{equation}\label{mixres_b}
\left\{
\begin{aligned}
\tau &= - \nabla v_i, \ \ &\mathrm{in} \ \Omega_i,\\
\nabla\cdot\tau &= f+\Delta \mathfrak{g}_i, \ \ &\mathrm{in} \ \Omega_i.
\end{aligned}
\right.
\end{equation}
{\ls The} mixed residual loss is
\begin{equation} \label{g_L_dd}
\small
\begin{aligned}
L(\tau_i,v_i,q) &= \int_{\Omega_i}[(\tau_i+\nabla v_i)^2+(\nabla\cdot \tau_i-f-\Delta \mathfrak{g}_i)^2]dxdy + q\int_{\partial\Omega_i} v_i^2 dxdy,\\
&\approx\frac{1}{m_1}\sum\limits_{k=1}^{m_1} [(\tau_i^{(k)}+\nabla v_i^{(k)})^2 + (\nabla\cdot\tau_i^{(k)} - f - \Delta\mathfrak{g}_i)^2] +\frac{1}{m_2}\sum\limits_{j=1}^{m_2}q \cdot(v_i^{(j)})^2.
\end{aligned}
\end{equation}
It should be noted that, the integration is completed by Monte Carlo, such that the domain decomposition reduces the variance of samples significantly with the same number of data because the area of samples becomes smaller.

\begin{algorithm}[H]
	\caption{Deep domain decomposition}
	\label{main}
	\begin{algorithmic}[1]
		\State {\bfseries Input:} $\Omega={\lt(x_0,x_1)\times(y_0,y_1)}$, $p$, $\Gamma_i$, $\eta$, $\theta$, $n$, $m_1$, $m_2$.
		\State {\bfseries Initialize:} $\epsilon = 10\times\eta$, $k=0$, $gv_i^0=\bm{0}$, $S_i$, $g_i$.
		\State Divide the physical domain $\Omega$ into $\Omega_1,\cdots,\Omega_p$.
		\While {$\epsilon>\eta$}
		\State Run Algorithm \eqref{left} in each subdomain in parallel.
		\State $\epsilon = \frac{1}{p}\sum\limits_{i=1}^{p}\Vert Sol_i^{(k+1)} - Sol_i^{(k)}\Vert_2^2$.
		\State $k = k + 1$.
		\EndWhile
		\State Merge $p$ parts $Sol_i^{(k)}$ and get $Dnn_{sol}^{(k)}$.
		\State {\bfseries Return:} $Dnn_{sol}^{(k)}$.
	\end{algorithmic}
\end{algorithm}

\begin{algorithm}[H]
	\caption{Training for subdomain $\Omega_i$}
	\label{left}
	\begin{algorithmic}[1]
		\State {\bfseries Input:} $S_i$, $gv_i^{k}$, $g_i^{k}$, $n$, $m_1$, $m_2$.
		\State Construct function  $\mathfrak{g}_i$ using value of $gv_i^{k}$, $v_i = \mb{N}_u - \mathfrak{g}_i$.
		\For{$n$ steps}
			\State Sample minibatch of $m_1$ samples $\hat{S}_i = \{(x_i, y_i)\}_{i=1}^{m_1}$ in $\Omega_i$.
			\State Sample minibatch of $m_2$ samples $\hat{g}_i = \{(x_i, y_i)\}_{i=1}^{m_2}$ on $\partial\Omega_i$.
			\State Update the parameters $\theta_i$ by descending its stochastic gradient: 
\begin{equation*}
\small
\begin{aligned}
\theta_i^{(k+1)} =  \theta_i^{(k)} -  
\nabla_{\theta} \frac{1}{m_1}\sum\limits_{k=1}^{m_1} [(\mb{N}_\tau^{(k)}+\nabla v_i^{(k)})^2 + (\nabla\cdot\mb{N}_\tau^{(k)} - f - \Delta\mathfrak{g}_i)^2] - \nabla_{\theta} \frac{1}{m_2}\sum\limits_{j=1}^{m_2}(q \cdot v_i^{(j)})^2.
\end{aligned}	
\end{equation*}			
		\EndFor
		\State $Sol_i^{(k+1)} = \mb{N}_u(S_i)$.
		\State $gv_i^{(k+1)} = \mb{N}_u(g_i^{(k)})$.
		\State {\bfseries Return:} $Sol_i^{(k+1)}$, $gv_i^{(k+1)}$.
	\end{algorithmic}
\end{algorithm}
The procedure of D3M is as follows. We first divide the domain $\Omega$ into $d$ subdomains, and each two {\ls neighboring} subdomains are overlapping. The {\ls local} solution of PDEs on each subdomain is replaced by neural networks which can be trained through the variational principle, where the {\ls global} solution on the whole domain consists of these local solutions on subdomains. To be more precise, let $\Gamma_i$ denote decomposed junctions, $\theta$ is initial weights of neural networks, $\eta$ is the threshold of accuracy, $S_i$ and $g_i$ are the samples generated in $\Omega_i$ and on interface $\Gamma_i$ to evaluate the output of networks in each iteration, $\hat{S}_i$ are training samples in subdomain $\Omega_i$, $\hat{g}_i$ are training samples on $\Gamma_i$, $n$ is training time in each iteration, $m_1$ and $m_2$ are batch sizes, $\bf{N}_u$ is the neural networks for $u$, $\bf{N}_\tau$ is the neural networks for $\tau$ , $k$ is the iteration time, and $Sol_i^{(k+1)}$ is the output of networks for subdomain $\Omega_i$ in {\ls $(k+1)$-th} iteration. The formal description of D3M {\ls is presented} in Algorithm \ref{main}.
\section{Analysis}\label{analysis}
{\ls While} mixed residual formulation is a special case, we consider the basic functional formulation first,
\begin{equation} \label{eq_vf_sec4}
J(u) = \int_{\Omega}\frac{1}{2}\nabla u\cdot\nabla u - fu dS.
\end{equation}
\begin{mdefinition}
Given $m$ closed subspaces $\{V_i\}_{i=1}^m$ and $V = \sum_{i=1}^{m}V_i \in C^2(\Omega)$, for any $R < \infty$, and a proper, lower semi-continuous, coercive convex functional $J: V \rightarrow \Re $, 
{\ls we denote} $K_R := \{u\in V|J(u)<R\}$.
\end{mdefinition}
\begin{massumption}\label{assum1}
	$J\in C^1(K_R)$ and $\exists\ \alpha_R>0$\ \ s.t. $\forall v,u\in K_R$
	\begin{equation}
	\begin{aligned}
	&J(v)-J(u)-(J'(u),v-u)\geq\alpha_R|v-u|^2,
	\end{aligned}
	\end{equation}
	where J' is uniformly continuous on $K_R$.
\end{massumption}
\begin{mproposition}\label{pro2}
$\mathrm{(P.L.Lions, 1989)}$ The alternating Schwarz method \eqref{decom1} and \eqref{decom2} converges to the solution of $u$ of Equation \eqref{mixres_b}. The error bound of $\hat{u_1}^{k+1}$ and $\hat{u_2}^{k+1}$ can be estimated via maximum principle \cite{kantorovich1960approximate,lions1989schwarz}, $\exists\ \rho\in(0,1)$ such that for $\forall\ k\geq0$
\begin{equation}
\Vert u|_{\Omega_i}-\hat{u_i}^{k+1}\Vert_{L^{\infty}(\Omega_i)}\leq \rho^k\Vert u|_{\Omega_i}-\hat{u}^{0}_{i}\Vert_{L^{\infty}(\Omega_i)}.
\end{equation}
where constants $\rho$ is close to one if the overlapping region $\Omega_{i,j}$ is thin.
\end{mproposition}
\begin{mlemma} \label{l1}
	$\mathrm{(K.\ Hornik, 1991)}$ On each subdomain $\Omega_i$, neural network $\mb{N}_i$ with continuous derivatives up to order $K$ are universal approximators in Sobolev space with an order $K$, which means $\mb{N}_i \in H^1(\Omega_i)$.
\end{mlemma}

\begin{mlemma} \label{l2}
	$\mathrm{(P.\ L.\ Lions, 1988)}$ If the variational formulation \eqref{eq_vf_sec4} satisfies the assumption \eqref{assum1}, then it follows that there exists a sequential $u_n\in V_i$ obtained by Schwarz alternating method converges to the minimum $u_i^*$ of $J(u_i)$ on each subdomain $\Omega_i$. 
\end{mlemma}

\begin{mtheorem}
	$J_{i}(\mb{N}_i)$ denotes the objective function on the subdomain $\Omega_i$. Under above assumptions, for $\forall\epsilon>0$, $\exists M>0$, while iteration times $k>M$, $\mb{N}_i^k$ converges to optimal solution $u_i^*$ of $J(u_i)$ in subdomain $\Omega_i$ for a constant $C>0$
	\begin{equation}
		|\mb{N}_i^k-u_i^*|^2\leq C\epsilon^{1/2} \ \ \mathrm{in}\ \Omega_i.
	\end{equation}
\end{mtheorem}
For concision, we use $\mb{N}_i$ to represent $\mb{N}_i^k$ in the following part.
\begin{proof}
	$u_n\in V_i$ denotes the sequential in Lemma \ref{l2}, there exist
	\begin{equation}\label{conver0}
		|u_n-u_i^*|\leq\frac{C_0}{\alpha_R}\omega|u_{n+1}-u_n|,
	\end{equation}
	where $C_0 > 0$ is a constant, $\omega|u_{n+1}-u_n|=|J'(u_{n+1})-J'(u_n)|$.\\
	Then with Lemma \ref{l1} \cite{hornik1991approximation}, in each subdomain $\Omega_i$, neural network $\mb{N}_i\in H^1$. If the training times $k$ in each subdomain is enough, the universal approximation ensures the distance between functions $\mb{N}_i$ and $u_n$ is close enough in the Sobolev space, with the constant $C_1>0$
	\begin{equation}\label{conver1}
		|\mb{N}_i-u_n|^2\leq C_1^2\epsilon^2.
	\end{equation}
	While $u_n$ converging to the minimum of $u_i^*$ of $J(u)$, by the optimality conditions it is clear that
	\begin{equation}
		\begin{aligned}
			(J'(u_{n+1}),u_{n+1}-u_n) < \epsilon,\\
			(J'(u_{n+1}),u_{n+1}-u_i^*) < \epsilon.
		\end{aligned}
	\end{equation} 
	Under the assumption, $\exists\ C>0$ and the difference between two iterations can be represented as
	\begin{equation}\label{conver2}
		{\lt|u_{n+1}-u_n|^2\leq \frac{1}{\alpha_R}|J(u_{n+1})-J(u_n)|\leq \frac{1}{\alpha_R}\epsilon}
	\end{equation}
	Consider equation \eqref{conver0}, \eqref{conver1} and \eqref{conver2}, we have
	\begin{equation}
		\begin{aligned}
			|\mb{N}_i - u_i^*| &\leq |\mb{N}_i-u_n| + |u_n-u_i^*|,\\
			&\leq C_1\epsilon + \frac{C_0}{\alpha_R}\omega|u_{n+1}-u_n|,\\
			&\leq C\epsilon^{1/2}.
		\end{aligned}
	\end{equation}
\end{proof}

\begin{mtheorem}
	For a given boundary function $\mathfrak{g}$ and a fixed $q$, the optimal solution $\mb{N}_u^*$ of Equation \eqref{eq_variational_nn} and $v^*$ of Equation \eqref{mixres_b} satisfy $\mb{N}_u^* = v^* + \mathfrak{g}$.
\end{mtheorem}

\begin{proof}
	We use $\tau$ and $\phi$ to denote $-\nabla\mb{N}_u$ and $-\nabla v$ respectively.
	\begin{equation}
	\small
		\begin{aligned}
			L(\phi,v)
			&=\int_{\Omega}[(\phi+\nabla v)^2+(\nabla\cdot\phi-f-\Delta\mathfrak{g})^2]dxdy + q\int_{\partial\Omega} v dxdy,\ \ \mathrm{if}\ v|_{\partial\Omega}=0\\
			&=\int_{\Omega}[(\tau+\nabla\mathfrak{g}+\nabla u-\nabla\mathfrak{g})^2+(\nabla\cdot\tau-f)^2]dxdy + q\int_{\partial\Omega} (u-\mathfrak{g}) dxdy\\
			&=L(\tau,u).
		\end{aligned}
	\end{equation}
	Function $v$ we optimized is also the optimal solution $\mb{N}_u$ with the formula $\mb{N}_u^*=v^*+\mathfrak{g}$.
\end{proof}

Up to now, we prove that D3M can solve steady Poisson's equation with variational formulations. Then, we extend D3M to more general quasilinear parabolic PDEs \eqref{quasi} with physics-constrained approaches.
\begin{equation}\label{quasi}
	{\lt\left\{
    \begin{aligned} &\operatorname{div}(\alpha( x, u_{i}( x), \tau_{i}(x))) + \gamma( x, u_{i}(x), \tau_{i}(t, x))=0, &x \in \Omega_{i} \\
    &u_{i}(t, x) =0, &x \in \partial \Omega_{i} \end{aligned}
    \right.}
\end{equation}
where $\tau_i$ denotes $\nabla u_i$, and $\Omega_{i}\in\mathbb{R}^d$ are decomposed boundary sets with smooth boundaries $\partial\Omega_{i}$. We recall the network space of subdomain $\Omega_{i}$ with generated data according to \cite{hornik1991approximation}
\begin{equation}
    \mathfrak{N}_i^n(\sigma)=\left\{h(x) : \mathbb{R}^{k} \rightarrow \mathbb{R} | h(x)=\sum_{j=1}^{n} \beta_{j} \sigma\left(\alpha_{j}^{T} x-\theta_{j}\right)\right\}
\end{equation}
where $\sigma$ is any activation function, $\mb{x}\in\mathbb{R}^{k}$ is one set of generated data, $\beta\in\mathbb{R}^n$, $\alpha\in\mathbb{R}^{k\times n}$ and $\theta\in\mathbb{R}^{k\times n}$ denote coefficients of networks. Set $\mathfrak{N}_i(\sigma)=\bigcup_{i=1}^{\infty}\mathfrak{N}_i^n(\sigma)$. Under the universal approximation of neural networks {\ls Lemma} \ref{l1}, in each subdomain the neural networks $f_i^n$ 
satisfies
{\lt
\begin{equation}
    \left\{
    \begin{aligned} &\operatorname{div}(\alpha( x, f^n_{i}( x), \tau^n_{i}(x)) + \gamma( x, f^n_{i}(x), \tau^n_{i}(t, x))=h^n, &x \in \Omega_{i} \\
    &f^n_{i}(t, x) =b^n, &x \in \partial \Omega_{i} \end{aligned}
    \right.
\end{equation}
}
where $h^n$ and $b^n$ satisfy
\begin{equation}\label{tozero}
    \Vert h^n\Vert_{2,\Omega_i}^2+\Vert b^n\Vert_{2,\partial\Omega_i}^2\to 0, \ \ as\ \ n\to\infty.
\end{equation}

For the following part of analysis, we make some assumptions.

\begin{massumption}\label{assum2}
    \begin{itemize}
    \item There is a constant $\mu > 0$ and positive functions $\kappa(x), \lambda(x)$ such that for all $x\in\Omega_i$ we have
    \begin{equation*}
    \|\alpha(x, u_i, \tau_i)\| \leq \mu(\kappa(x)+\|\tau_i\|),
    \end{equation*}
    and
    \begin{equation*}
    |\gamma(x, u_i, \tau_i)| \leq \lambda(x)\|\tau_i\|,
    \end{equation*}
    with $\kappa \in L^{2}\left(\Omega_{i}\right), \lambda \in L^{d+2+\eta}\left(\Omega_{i}\right)$ for some $\eta>0$.
    
    \item $\alpha(x,u_i,\nabla u_i)$ and $\gamma(x,u_i,\nabla u_i)$ are Lipschitz continuous in $(x, u_i,\nabla u_i) \in \Omega \times \mathbb{R} \times \mathbb{R}^{d}$.

    \item In each subdomain, the derivatives of solutions from alternating Schwarz method \eqref{decom1}, \eqref{decom2} converge to the derivative of solution $u_i$. Precisely, there exists a constant $\rho_1\in(0,1)$, such that for iteration times $\forall k\geq0$
    \begin{equation*}
    	\Vert\nabla u^*_i-\nabla u^k_i\Vert_{\infty}\leq\rho_1^k\Vert\nabla u^*_i-\nabla u^0_i\Vert_{\infty}.
    \end{equation*}
    
    \item $\alpha(x, u, \tau)$ is continuously differentiable w.r.t. $(t,x)$.
    
    \item There is a positive constant $ \nu > 0$ such that
    \begin{equation*}
    \alpha(x, u, \tau) \tau \geq \nu|\tau|^{2}
    \end{equation*}
    and \ $\forall \tau_1,\tau_2\in\mathbb{R}^d, \tau_1\neq\tau_2$
    \begin{equation*}
    \left\langle\alpha\left(x, u, \tau_{1}\right)-\alpha\left(x, u, \tau_{2}\right), \tau_{1}-\tau_{2}\right\rangle> 0.
    \end{equation*}

\end{itemize}
\end{massumption}

\begin{mtheorem}
    Suppose the domain $\Omega$ is decomposed into $\{\Omega_{i}\}_{i=1}^p$, $k>0$ denotes iteration times (omitted in notations for brief). $\mathfrak{N}_i(\psi)$ denotes networks space space in subdomain $\Omega_{i}$, where subdomains are compact. Assume that target function \eqref{quasi} has unique solution in each subdomain, nonlinear terms $div(x,u,\nabla u)$ and $\gamma(x,u,\nabla u)$ are locally Lipschitz in ($u_i,\nabla u_i$), and $\nabla u^k_i$ uniformly converges to $\nabla u^*_i$ with $k$. For $\forall \ \epsilon > 0$, there $\exists \ K>0$ such that there exists a set of neural networks $\{\mb{N}_i\in\mathfrak{N}_i(\psi)\}_{i=1}^p$ satisfies the $L^2$ error $E_2(\mb{N}_i)$ as follow
    \begin{equation}
        \sum_{i=1}^p \lim_{k\to\infty}E_2(\mb{N}_i) \leq {\lt K\epsilon}.
    \end{equation}
\end{mtheorem}
\begin{proof}
	In each subdomain $\Omega_i$, with iteration times $k>0$, $E_2^k(\mb{N}_i)$ denotes the $L^2$ loss between $\mb{N}_i^k$ and $u_i^*$.
	\begin{equation}
		\begin{aligned}
			\lim_{k\to\infty}E_2^k(\mb{N}_i)
			&=\Vert\operatorname{div}(\alpha(x, u^*_{i}( x), \nabla u^*_{i}(x)))+\gamma(x, u^*_{i}(x),\nabla u^*_{i}(x)) \\
			& - [\operatorname{div}(\alpha(x, \mb{N}_{i}( x), \nabla \mb{N}_{i}( x)))+\gamma( x, \mb{N}_{i}( x), \nabla \mb{N}_{i}( x))]\Vert^2_{\Omega_{i}}
			 + \Vert \mb{N}_i\Vert^2_{\partial\Omega_{i}}
		\end{aligned}
	\end{equation}
	With Lemma 1, it is clear that the sum of last term is smaller than $K_1\epsilon$, where $K_1>0$ is a constant. We assumed that $\nabla u^k_i$ uniformly converges to $\nabla u^*_i$ with $k$, and this means that
	\begin{equation}
		\Vert\nabla u^*_i-\nabla u^k_i\Vert_{\infty}\leq\rho_1^k\Vert\nabla u^*_i-\nabla u^0_i\Vert_{\infty}.
	\end{equation}
	So that we have
	\begin{equation}
	\footnotesize
		\begin{aligned}
			\lim_{k\to\infty}E_2^k(\mb{N}_i)
			& \leq \int_{\Omega_{i}}|\operatorname{div}(\alpha( x, u^*_{i}( x), \nabla u^*_{i}( x))) - \operatorname{div}(\alpha( x, u^k_{i}( x), \nabla u^k_{i}( x)))
			+ \operatorname{div}(\alpha( x, u^k_{i}( x), \nabla u^k_{i}( x))) - \operatorname{div}(\alpha( x, \mb{N}_{i}( x), \nabla \mb{N}_{i}( x)))|dx\\
			& + \int_{\Omega_{i}}|\gamma( x, u^*_{i}( x), \nabla u^*_{i}( x)) - \gamma( x, u^k_{i}( x), \nabla u^k_{i}( x))
			+ \gamma( x, u^k_{i}( x), \nabla u^k_{i}( x)) - \gamma( x, \mb{N}_{i}( x), \nabla \mb{N}_{i}( x)) |dx + K_1\epsilon\\
			& \leq \int_{\Omega_{i}}|\operatorname{div}(\alpha( x, u^*_{i}( x), \nabla u^*_{i}( x))) - \operatorname{div}(\alpha( x, u^k_{i}( x), \nabla u^k_{i}( x)))|dx
			+ \int_{\Omega_{i}}|\operatorname{div}(\alpha( x, u^k_{i}( x), \nabla u^k_{i}( x))) - \operatorname{div}(\alpha( x, \mb{N}_{i}( x), \nabla \mb{N}_{i}( x)))|dx\\
			& + \int_{\Omega_{i}}|\gamma( x, u^*_{i}( x), \nabla u^*_{i}( x)) - \gamma( x, u^k_{i}( x), \nabla u^k_{i}( x))|dx
			+ \int_{\Omega_{i}}|\gamma( x, u^k_{i}( x), \nabla u^k_{i}( x)) - \gamma( x, \mb{N}_{i}( x), \nabla \mb{N}_{i}( x)) |dx + K_1\epsilon\\
			& \leq \int_{\Omega_{i}}|\operatorname{div}(\alpha( x, u^*_{i}( x), \nabla u^*_{i}( x))) - \operatorname{div}(\alpha( x, u^k_{i}( x), \nabla u^k_{i}( x)))|dx
			+ \int_{\Omega_{i}}|\gamma( x, u^*_{i}( x), \nabla u^*_{i}( x)) - \gamma( x, u^k_{i}( x), \nabla u^k_{i}( x))|dx + K_2\epsilon + K_1\epsilon
		\end{aligned}
	\end{equation}
	where $K_2>0$ is a constant, and the error bound between $u_i^k$ and $\mb{N}_i^k$ is proved from Theorem 7.1 in \cite{sirignano2018dgm}.
	
	\begin{equation}\label{div}
		\begin{aligned}
			&\lim_{k\to\infty}\int_{\Omega_{i}}|\operatorname{div}(\alpha( x, u^*_{i}( x), \nabla u^*_{i}( x))) - \operatorname{div}(\alpha( x, u^k_{i}( x), \nabla u^k_{i}( x)))|^2dx\\
			&\leq\int_{\Omega_{i}}(|u^k_i|^{q_1}+|\nabla u^k_i|)|^{q_2}+|u^*_i|^{q_3}+|\nabla u^*_i|)|^{q_4})\times(|u^*_i-u^k_i|+|\nabla u^*_i-\nabla u^k_i|)dx\\
			&\leq (\int_{\Omega_i}(|u^k_i-u^*_i|^{q_1}+|\nabla u^k_i-\nabla u^*_i|^{q_2}+|u^*_i|^{\max\{q_1,q_3\}}+|\nabla u^*_i|^{\max\{q_2,q_4\}})^r dx)^{1/r}\times(\int_{\Omega_i}\rho^k\Vert u^*_i-u^0_i\Vert_{\infty}+\rho^k_1\Vert \nabla u^*_i-\nabla u^0_i\Vert_{\infty}dx)\\
			&\leq (\int_{\Omega_{i}}\Vert u^k_i-u^*_i\Vert^{q_1}_{\infty}+\Vert \nabla u^k_i-\nabla u^*_i\Vert^{q_2}_{\infty}+\sup_{\Omega_i}|u^*_i|^{\max\{q_1,q_3\}}+\sup_{\Omega_i}|\nabla u^*_i|^{\max\{q_2,q_4\}}dx)\epsilon \\
			&\leq K_3 \epsilon
		\end{aligned}
	\end{equation}
	where $K_3>0$ is a constant depends on $\epsilon$. While $\gamma(\cdot)$ is also Lipschitz continuous, we can prove the upper bound of $\int_{\Omega_{i}}|\gamma( x, u^*_{i}( x), \nabla u^*_{i}( x)) - \gamma( x, u^k_{i}( x), \nabla u^k_{i}( x))|dx$ in the same formula with Equation \eqref{div}, denoted as $K_4\epsilon$.
	Hence we can obtain
	\begin{equation}
		\begin{aligned}
			\sum_{i=1}^p \lim_{k\to\infty}E_2(\mb{N}_i) 
			&\leq \sum_{i=1}^{p} (K_1\epsilon + K_2\epsilon + K_3\epsilon + K_4\epsilon) 
			\leq {\lt K\epsilon}
		\end{aligned}
	\end{equation}
\end{proof}

\begin{mtheorem}
Under Assumption \ref{assum2} and Equation \eqref{tozero}, with iteration times $k \to \infty$, the set of neural networks $\mb{N}_i$ converge to the unique solutions to \eqref{quasi}, strongly in $\mathcal{L}^{\rho}(\Omega_{i})$ for every $\rho<2$. In addition, in each subdomain the sequence $\{\mb{N}_i^n(x)\}_{n\in\mathbb{N}}$ is bounded in n under the constraint of $\mathrm{Proposition}$ \ref{pro2} and converges to $u_i$.
\end{mtheorem}

\begin{proof}
	In each subdomain, the convergence can be obtained from the Theorems 7.1 and 7.3 in \cite{sirignano2018dgm}. With the Proposition \ref{pro2}, the sequence $\{\mb{N}_{i}^{n}\}_{n\in\mathbb{N}}$ is uniform bounded in n, and the rates of convergence to the solution $u^*$ are related to overlapping areas.
\end{proof}

Specially, the $n$ in $\{\mb{N}_i^n(x)\}_{n\in\mathbb{N}}$ means training times instead of iteration times.We leave the proof for time-dependent data-free variational formulations for future work.
\section{D3M summary}\label{d3m_summary}
To summarize the strategies in Section \ref{d3m} and \ref{analysis}, the full procedure of D3M comprises the following steps:
\begin{itemize}
	\item pre-step: Set the architecture of neural networks in each subdomain;
	
	\item offline step: Construct functions for boundary conditions, train networks and generate local solutions; 
	
	\item online step: Estimate target input data using neural networks. If solutions don't converge, transfer information on interfaces and go back to the offline step.
\end{itemize}

The pre-step is cheap, because the setting of neural networks is an easy task. In the offline stage, the complex system is fully decomposed into component parts, which means that there is no data exchange between subdomains. Since we only approximate the data on interface with normal simple approach such as Fourier series and polynomial chaos \cite{chen2015local,arngha11a}, the approximation is also low costly. After decomposition, the requirement for number of samples for the Monte Carlo integration of  the residual loss function \eqref{g_L_dd} is significantly reduced, while the density of samples does not change. Since the number of samples decreasing and the domain becoming simpler, we can use neural networks with few layers to achieve a {\ls relatively high} accuracy. If we keep the same number of samples and layers as the global setting, D3M should obtain a better accuracy. The cost of the online step is low, since no PDE solve is required. This full approach is also summarized in Figure \ref{fig_full_approach}, where transformations of data happen between adjacent subdomains.
\vspace{1cm}
\begin{figure}
	\centering
	\includegraphics[width=0.8\linewidth]{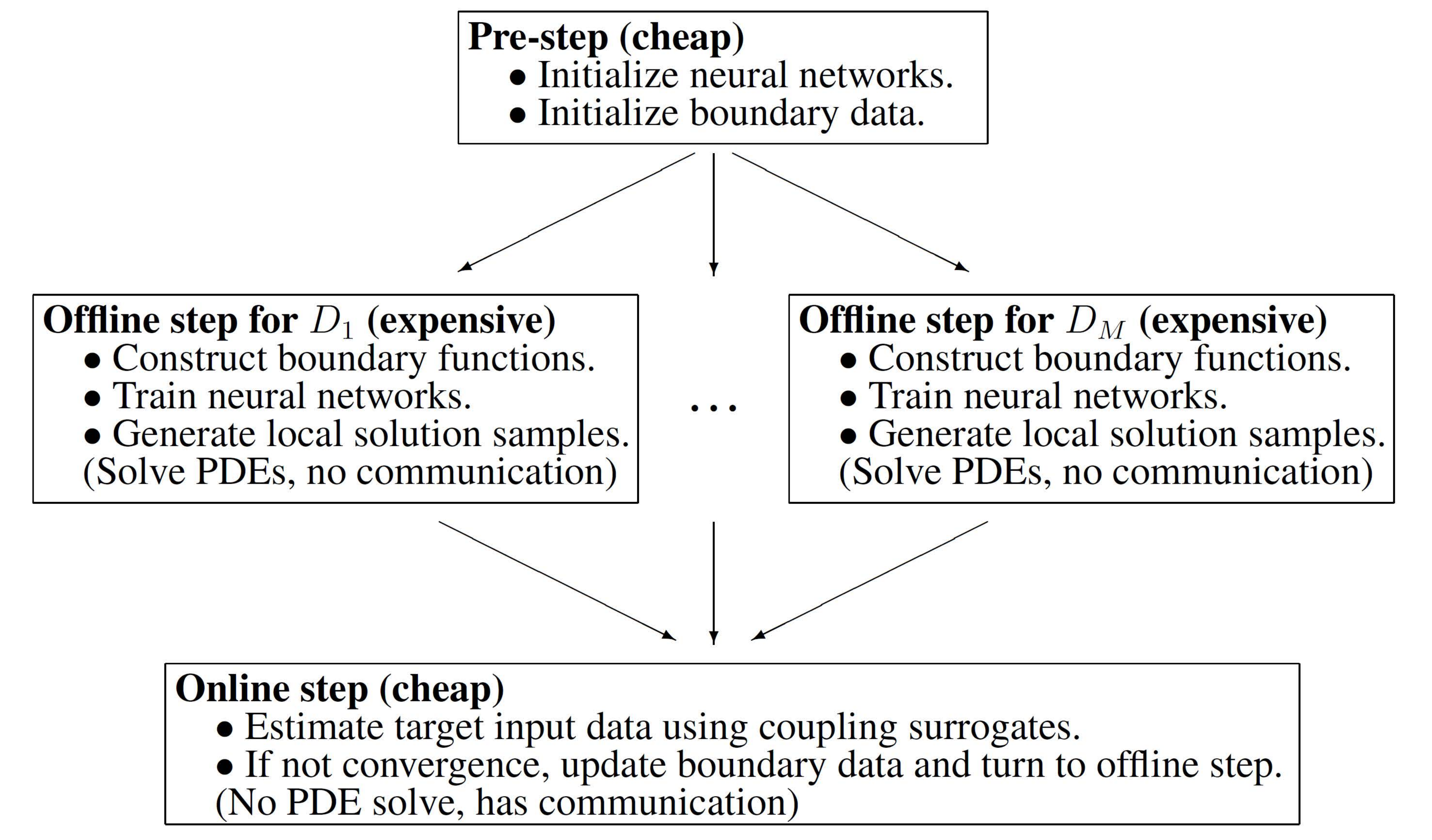}
	\caption{D3M summary.}
	\label{fig_full_approach}
\end{figure}

For the problems we focus on (systems governed by PDEs), the cost of the D3M is
dominated by the local training procedures in the offline step. Here we present a rough order of magnitude analysis of the computational costs where $C_{\textrm{solve}}$ denotes the cost of one block in each training epoch (i.e., the cost of any block with the same number of neurons in each training iteration is taken to be equal for simplicity). The dominant cost of D3M is the total number of blocks of neural network and training times, $\sum^{p}_{i=1}N_{i}B_{i}T_{i}C_{\textrm{solve}}$, where $N_{i}$ is sample size, $B_{i}$ is number of blocks and $T_{i}$ is training times. If we consider equal offline sample sizes, number of blocks and training epochs, $N_i=N_{\rm{off}}, B_i=B_{\rm{off}}, T_i=T_{\rm{off}}$ for all subdomains $\{D_i\}^p_{i=1}$, then total cost can be written as $p N_{\rm{off}}B_{\rm{off}}T_{\rm{off}}C_{\textrm{solve}}$. The total cost is decreased by employing the idea of hierarchical neural networks \cite{ng2014multifidelity,li2019hnh}. 


\section{Numerical tests}\label{tests}
Here we consider two classical problems, the Poisson's equation and
the  time-independent Schrödinger equation, to verify the performance of our D3M. All
timings conduct on an Intel Core i5-7500, 16GB RAM, Nvidia GTX 1080Ti processor with PyTorch 1.0.1 \cite{paszke2017automatic} under Python 3.6.5. We train the networks only 30 epoch using L-BFGS in numerical tests (cost within two minutes).
\subsection{Poisson's equation}
\begin{equation}
\left\{
\begin{aligned}
    -\Delta u(x,y) &= 1, \ \ &\mathrm{in} \ \Omega,\\
    u(x,y) &= 0, \ \ &\mathrm{on} \ \partial\Omega,\\
\end{aligned}
\right.
\end{equation}
where the physical domain {\ls is $\Omega=(-1,1)\times(-1,1)$}. 
{\ls The domain decomposition setting is illustrated in Figure  \ref{dd}.}
To further improve the efficiency of D3M, we propose a new type of sampling methods. 
We randomly sample in each subdomain,
and the number of samples increases with iteration times increase. The {\ls sample} size on interfaces 
remains the same to provide an accurate solution for data exchange. An illustration of our  D3M sampling method is represented in Figure \ref{d3msampling}
\begin{figure}[H]
	\centering
	\includegraphics[width=0.4\linewidth]{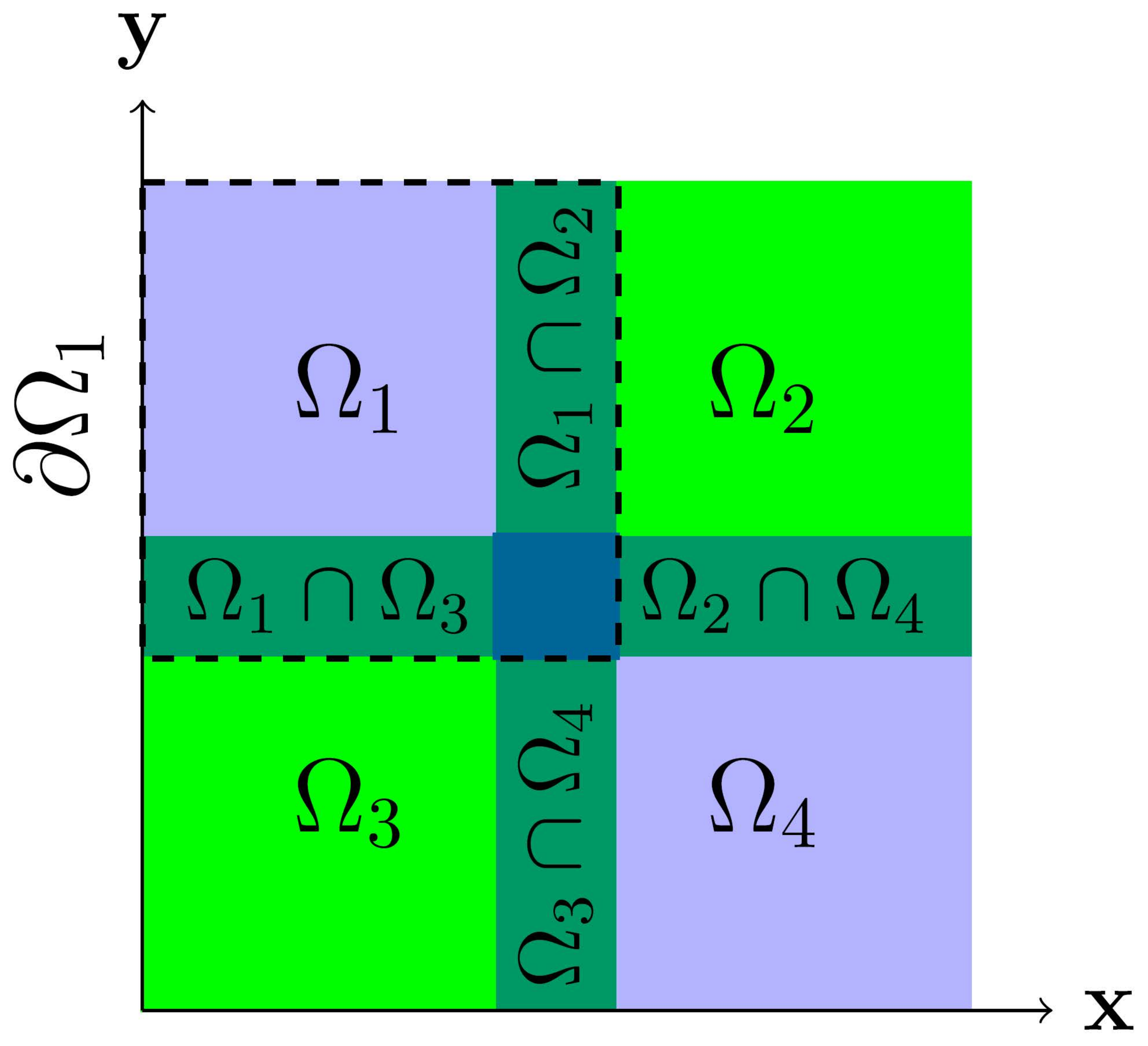}
	\caption{Illustrations of the physical domain with four overlapping components.}
	\label{dd}
\end{figure}

\begin{figure}
	\centering
	\subfigure[]{
		\centering
		\includegraphics[width=0.35\linewidth]{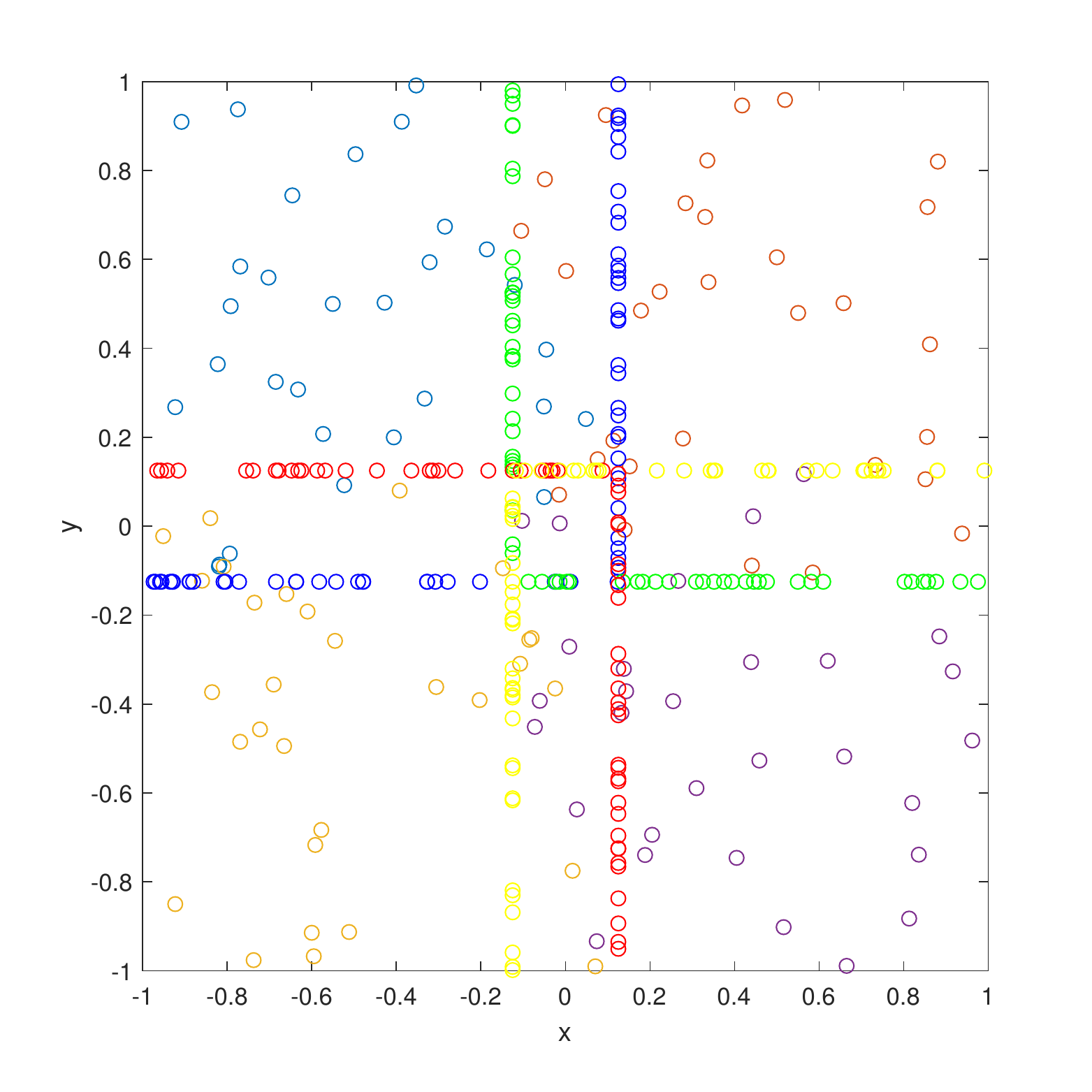}}
	\subfigure[]{
		\centering
		\includegraphics[width=0.35\linewidth]{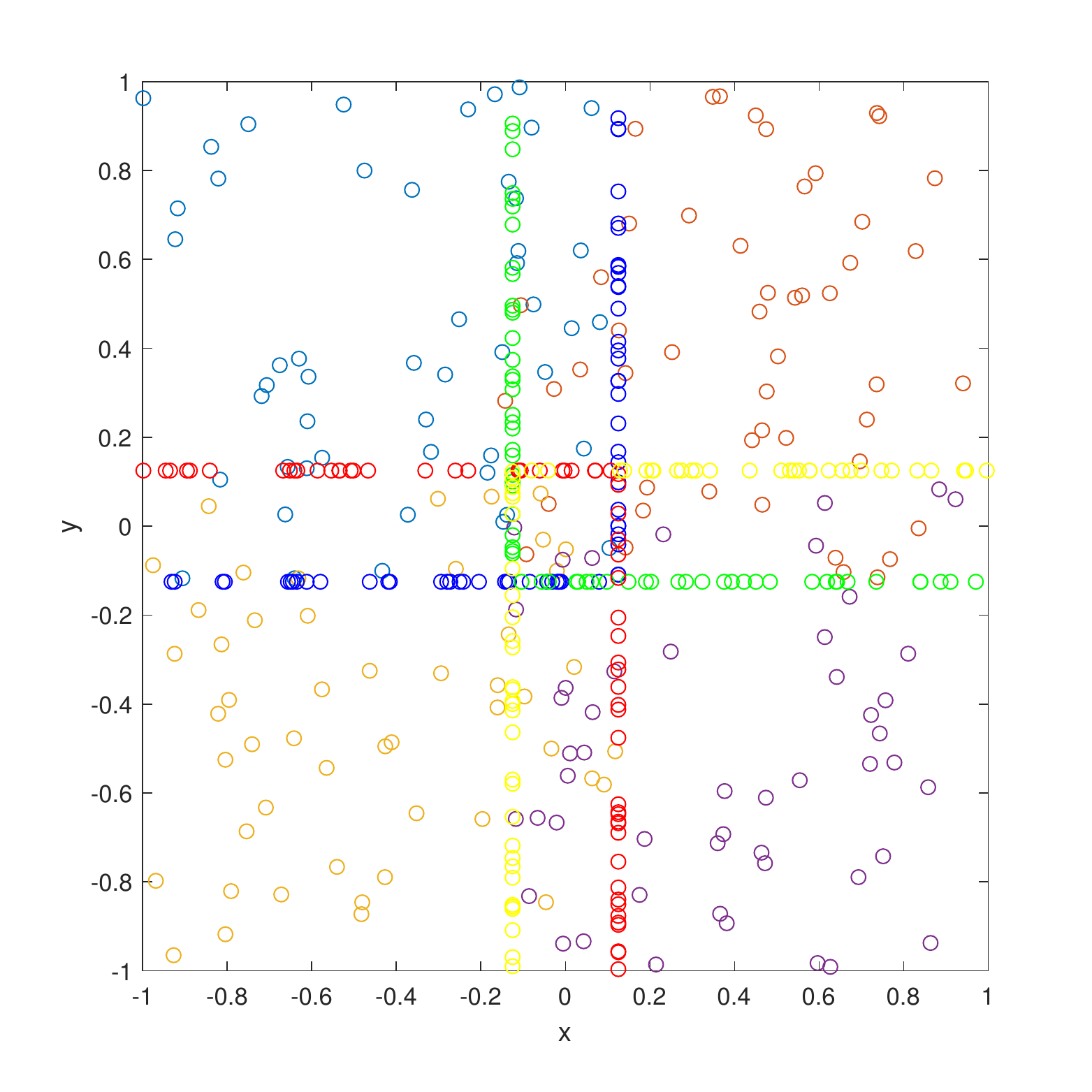}}
	\subfigure[]{
		\centering
		\includegraphics[width=0.35\linewidth]{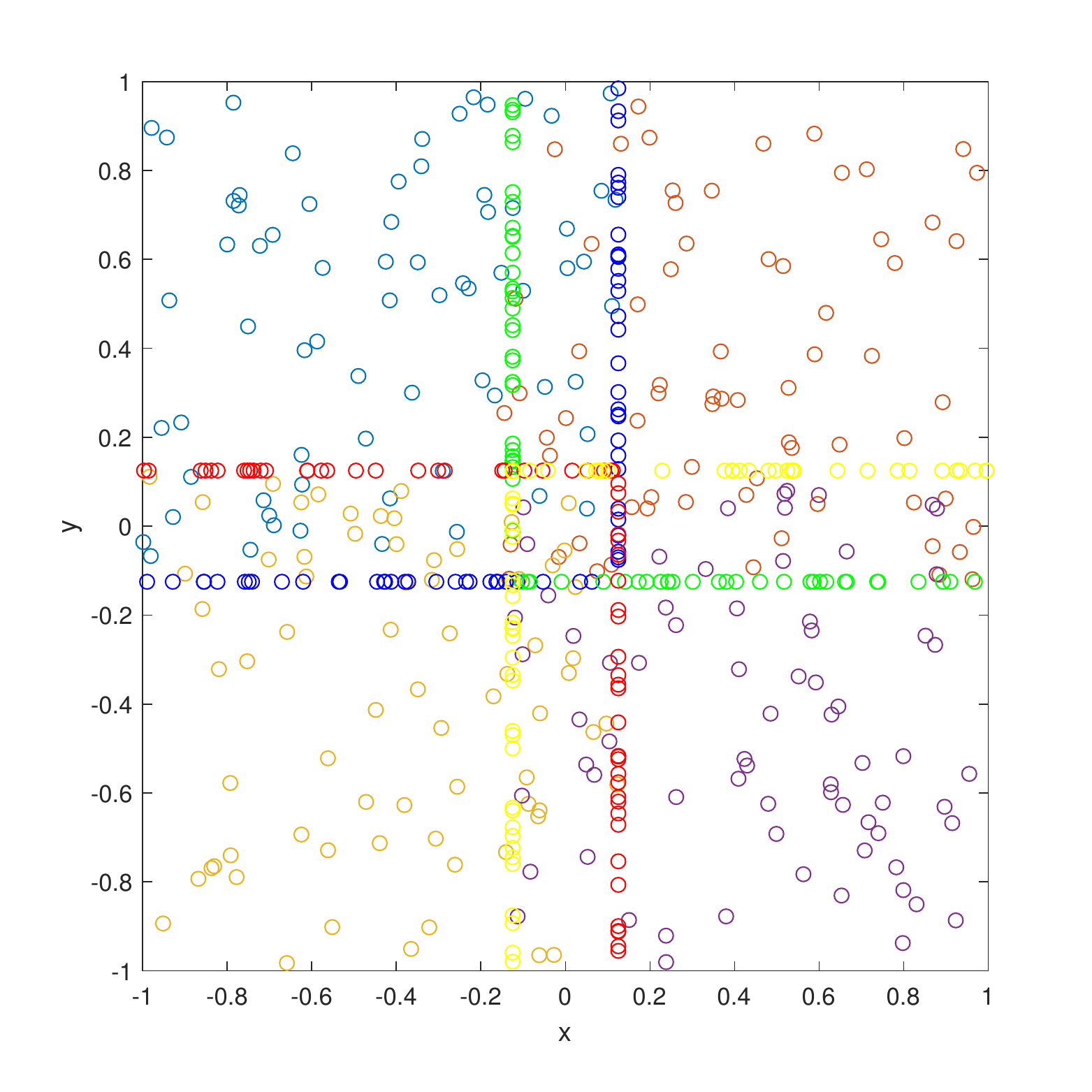}}
	\subfigure[]{
		\centering
		\includegraphics[width=0.35\linewidth]{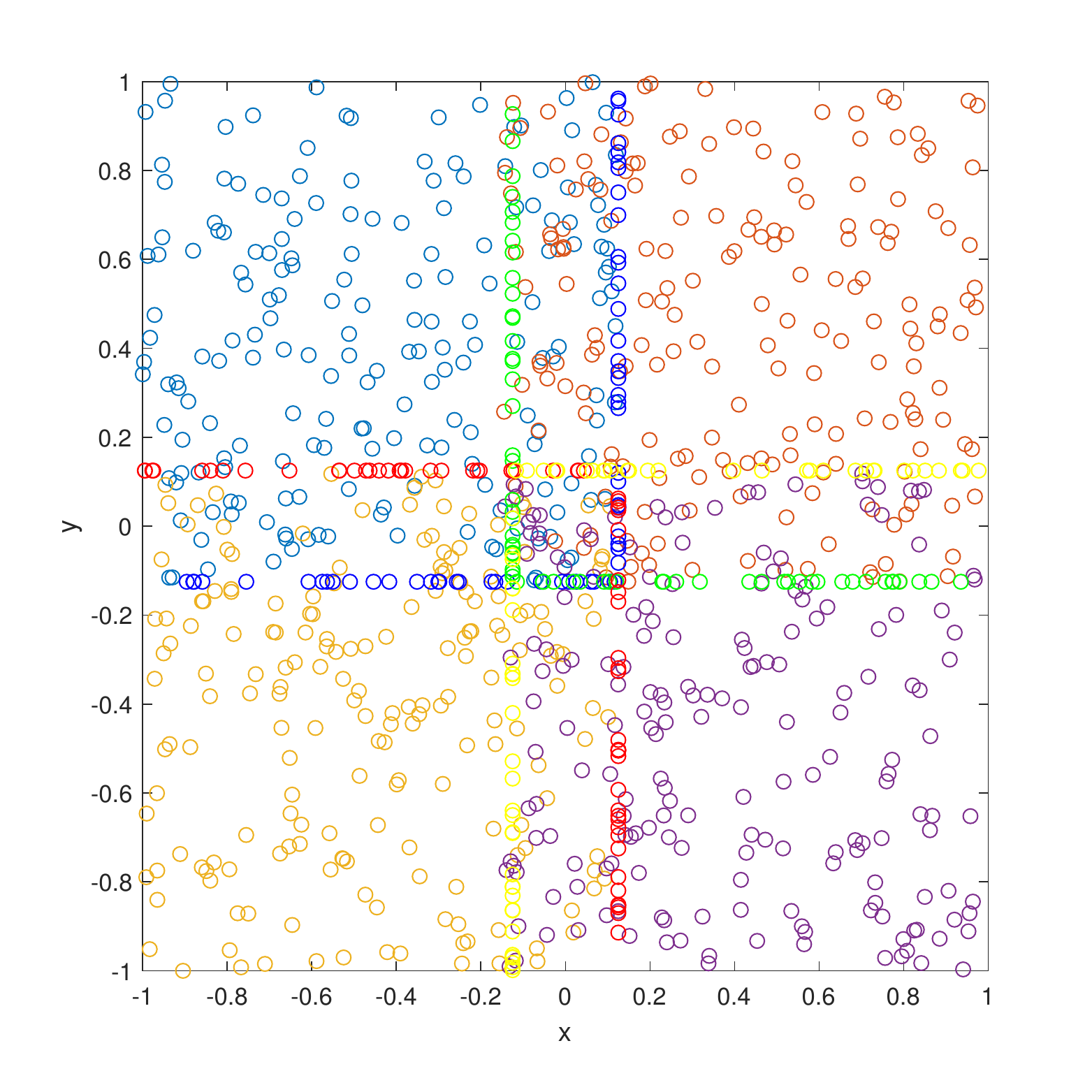}}

\caption{D3M sampling : a new type of mesh-free sampling.}
\label{d3msampling}
\end{figure}

\begin{mremark}
	According to the research on overfitting, the hypothesis set (w.r.t. the complexity of neural networks) should match the quantity and quality of data instead of target functions \cite{abu2012learning}. So in the initial several iterations, the number of residual blocks is small. The number increases while the sample size in Figure \ref{d3msampling} {\ls increases}.
\end{mremark}

After decomposition, with designed $\mathfrak{g}_i$ satisfying Definition 1 for each subdomains, the function $v_i:= u_i-\mathfrak{g}_i$ satisfies the homogeneous Poisson's equation and follows \eqref{decom1} and \eqref{decom2}.

\begin{figure}
	\centering
	\subfigure[Solution of D3M]{
		\centering
		\includegraphics[width=0.4\linewidth]{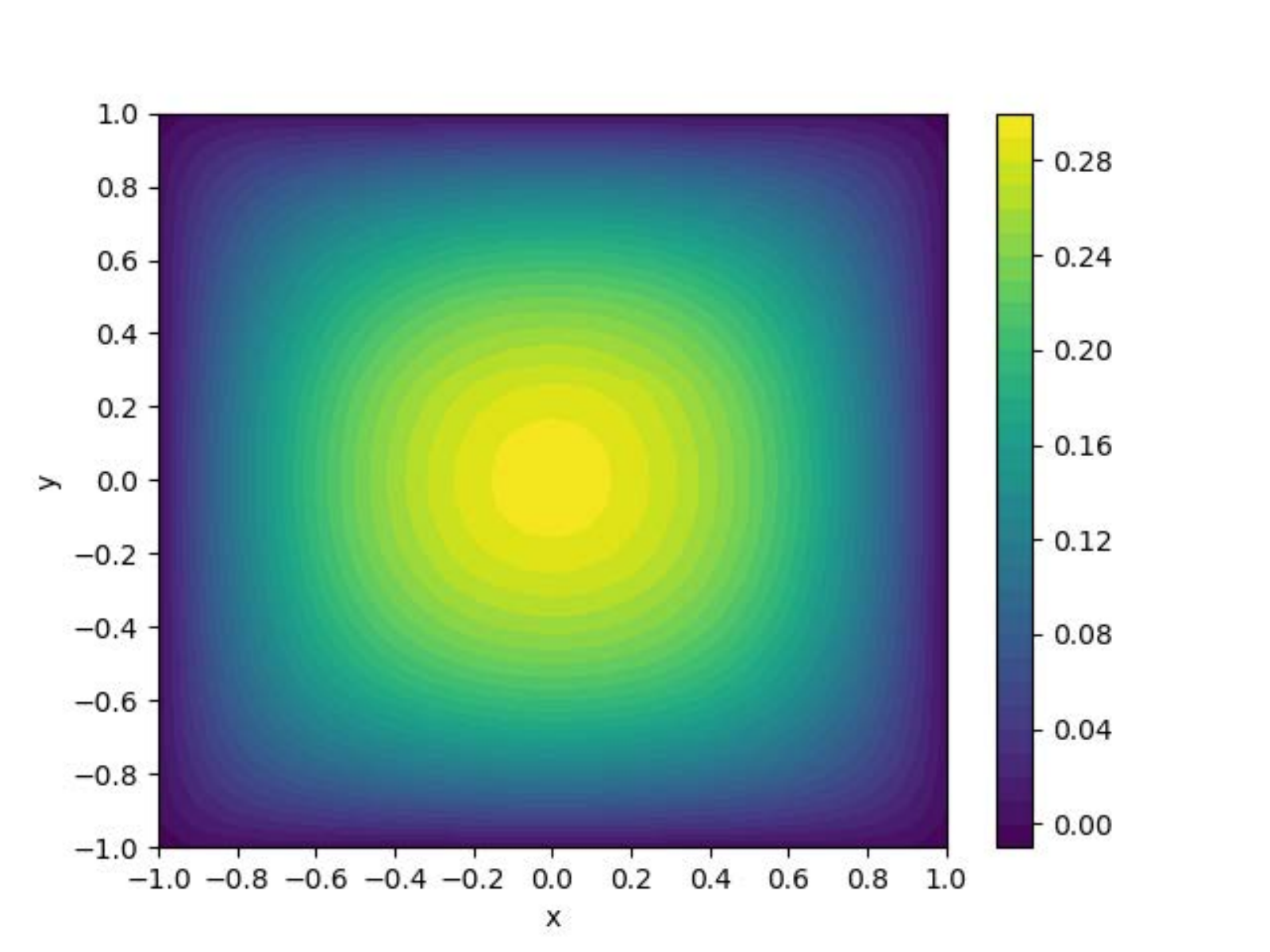}
	\label{1.1}}
	\subfigure[Solution of DRM]{
		\centering
		\includegraphics[width=0.4\linewidth]{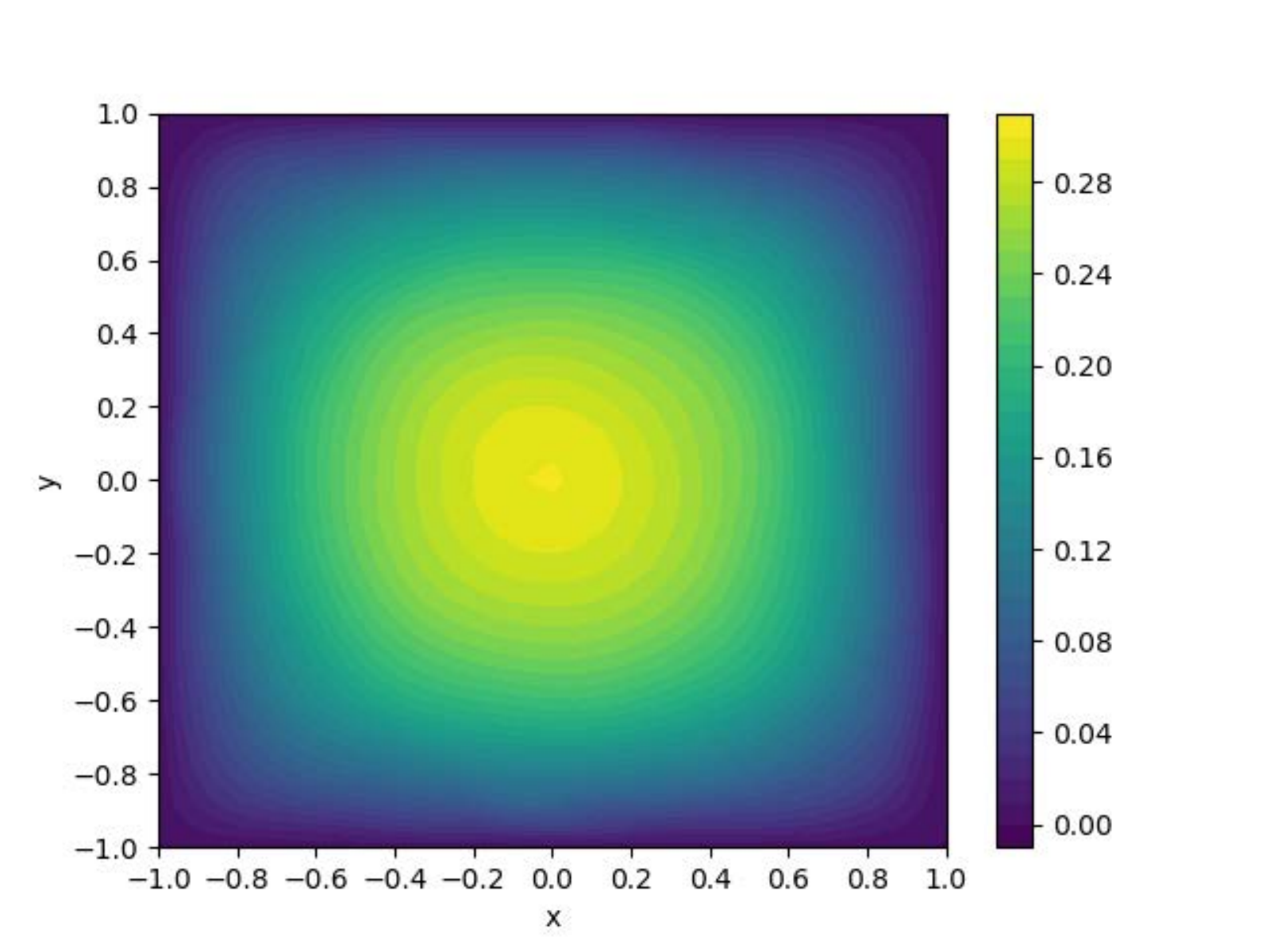}
	\label{1.2}}
	\subfigure[Solution of FEM]{
		\centering
		\includegraphics[width=0.4\linewidth]{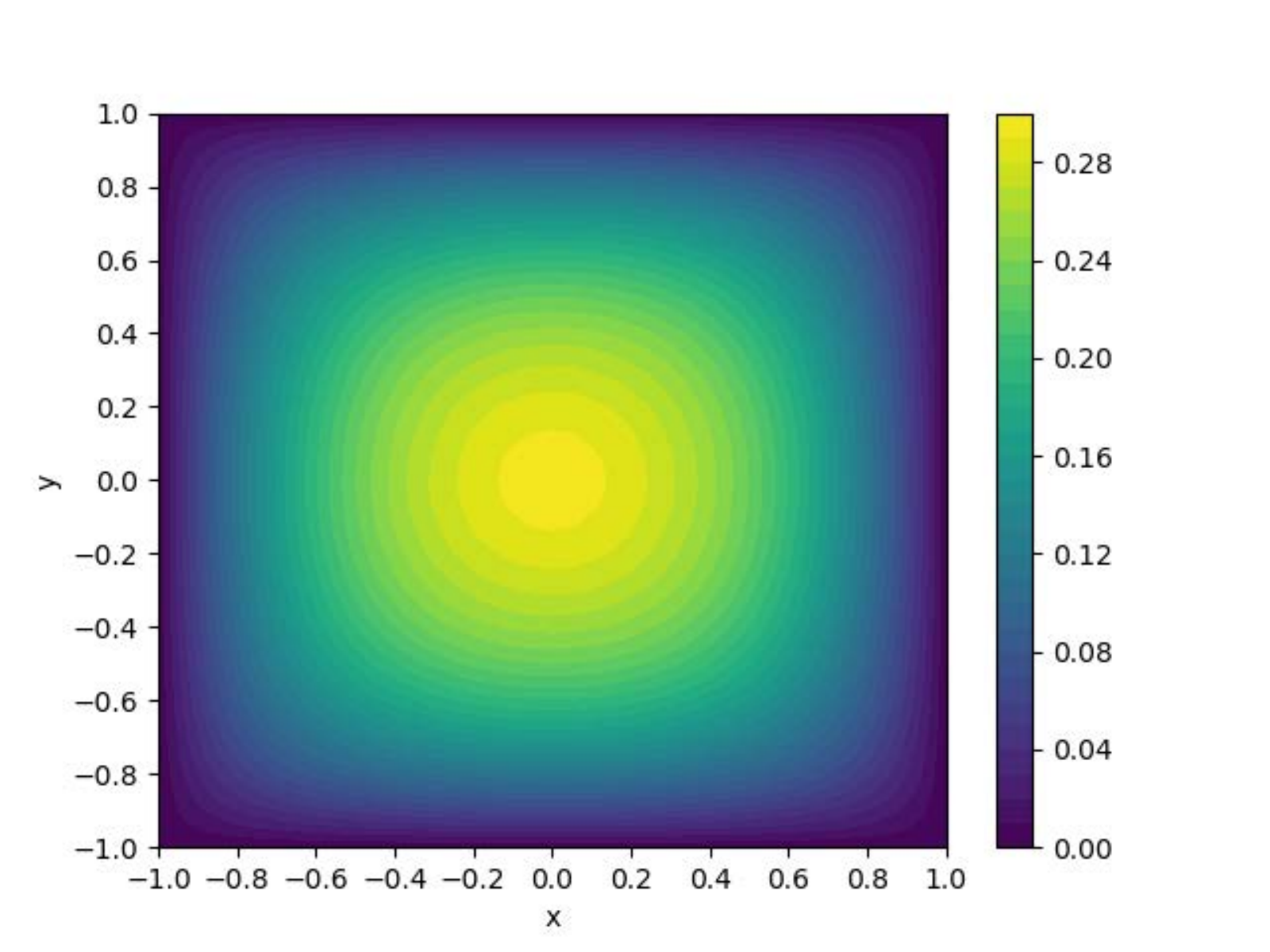}
	\label{1.3}}
	\caption{Solutions computed by three different methods.}
	\label{comp}
\end{figure}

The results of D3M is shown in \ref{1.1}. For comparison, we plot the result of normal deep Ritz method (DRM) with the same type of network and the finite element method (FEM) in Figure \ref{1.2} and \ref{1.3}. We set the result of FEM as the groundtruth and define the relative error $e_r=\frac{\norm{sol-fem_{sol}}{2}}{\norm{fem_{sol}}{2}}$. We compare results using residual network (ResNet), and the comparison including relative errors are shown in Table \ref{table1}. We can see that with the same setting for networks, our D3M offers a higher accuracy than normal DRM in this experiment.

\begin{table}[H]
	\centering
	\begin{threeparttable}
		\caption{The relative error for Poisson's equation.}
		\label{table1}
		\begin{tabular}{ccccc}
			\hline
			Method\ &Net type \ &Blocks \ &Number of neurons\ &Relative error\\
			\hline
			DRM&ResNet&4&2048&0.0271\\
			DRM&ResNet&8&4096&0.0157\\
			D3M&ResNet&4&2048&0.0065\\
			D3M&ResNet&8&4096&0.0045\\
			\hline
		\end{tabular}
	\end{threeparttable}
\end{table}

\subsection{Schrödinger equation}
In the area of domain decomposition methods, the steady-state Schrödinger equation is one of the classical problems \cite{toselli2006domain,hagstrom1988numerical}.
\begin{equation}
    \left[\frac{-\bar{h}^{2}}{2 m} \nabla^{2}+V(\mathbf{r})\right] \Psi(\mathbf{r})=E \Psi(\mathbf{r}).
\end{equation}
Where $\bar{h}=\frac{h}{2 \pi}$ is the reduced Planck constant, $m$ is the particle's mass and $E$ is a known constant related to the energy level. This equation occurs often in quantum mechanics where $V(\mathbf{r})$ is the function for potential energy. Here we consider an infinite potential well
\begin{equation}
    V(\mathbf{r})=\left\{\begin{array}{cc}{0,} \quad &{\mathbf{r} \in[0,1]^{d}} \\ {\infty,} \quad &{\mathbf{r} \notin[0,1]^{d}}.\end{array}\right.
\end{equation}
The variational loss is
\begin{equation}
\begin{aligned}
    &L(\tau_i,\mb{N}_i,q) = \int_{\Omega_i}[\Delta\mb{N}_i(r)) + \Delta\mathfrak{g}_i - E\mb{N}_i(r)) + E\mathfrak{g}_i]dr \\
    &+ q\int_{\partial\Omega_i}(\mb{N}_i(r))^2 dr + \gamma(\int_{\Omega_i}|\mb{N}_i(r)|^2dr-1+P_i)^2,
\end{aligned}
\end{equation}
where $\int_{\Omega}|\Psi(r)|^2dr=1$, because $\Psi(r)$ is {\ls the} wave function and $\Psi(r)^2$ means 
{\ls the} probability density of particle appearing. $P_i$ denotes the probability of particle appearing in $\Omega \backslash \Omega_i$. 
{\ls It should be noted that $\mb{N}_i$ is an approximation of $\Psi(\mathbf{r})_i-g_i$,
where  $\Psi(\mathbf{r})_i$ is the global solution $\Psi(\mathbf{r})$ restricted on the subdomain $\Omega_i$
and $g_i$ is the boundary function for the interface of $\Omega_i$.}

For this two dimensional time-independent Schrödinger equation, we can calculate the analytical solution $\Psi(r)=A\sin(\frac{m\pi r_1}{b})C\sin(\frac{n\pi r_2}{a})$, where $A=C=\sqrt{2}$, $a=b=1$ and $n=m=1$ in this infinite potential well case with domain $[0,1]\times[0,1]$. As shown in Figure \ref{scho_fig}, excellent agreement can be achieved between the exact solutions and predictions from our D3M. Compared with the solutions of DRM, D3M shows a better performance especially in peak values. The comparison of accuracy for both the wave equation and the probability density is shown in Table \ref{table2} and Figure \ref{schro_comp}. Under the same conditions including network structure, Lagrangian multiplier, learning rate, number of samples and training epoch, our D3M shows smaller errors in both wave equation and probability density.

\begin{figure}
	\centering
	\subfigure[Exact solution of $\Phi(r)$]{
		\centering
		\includegraphics[width=0.4\linewidth]{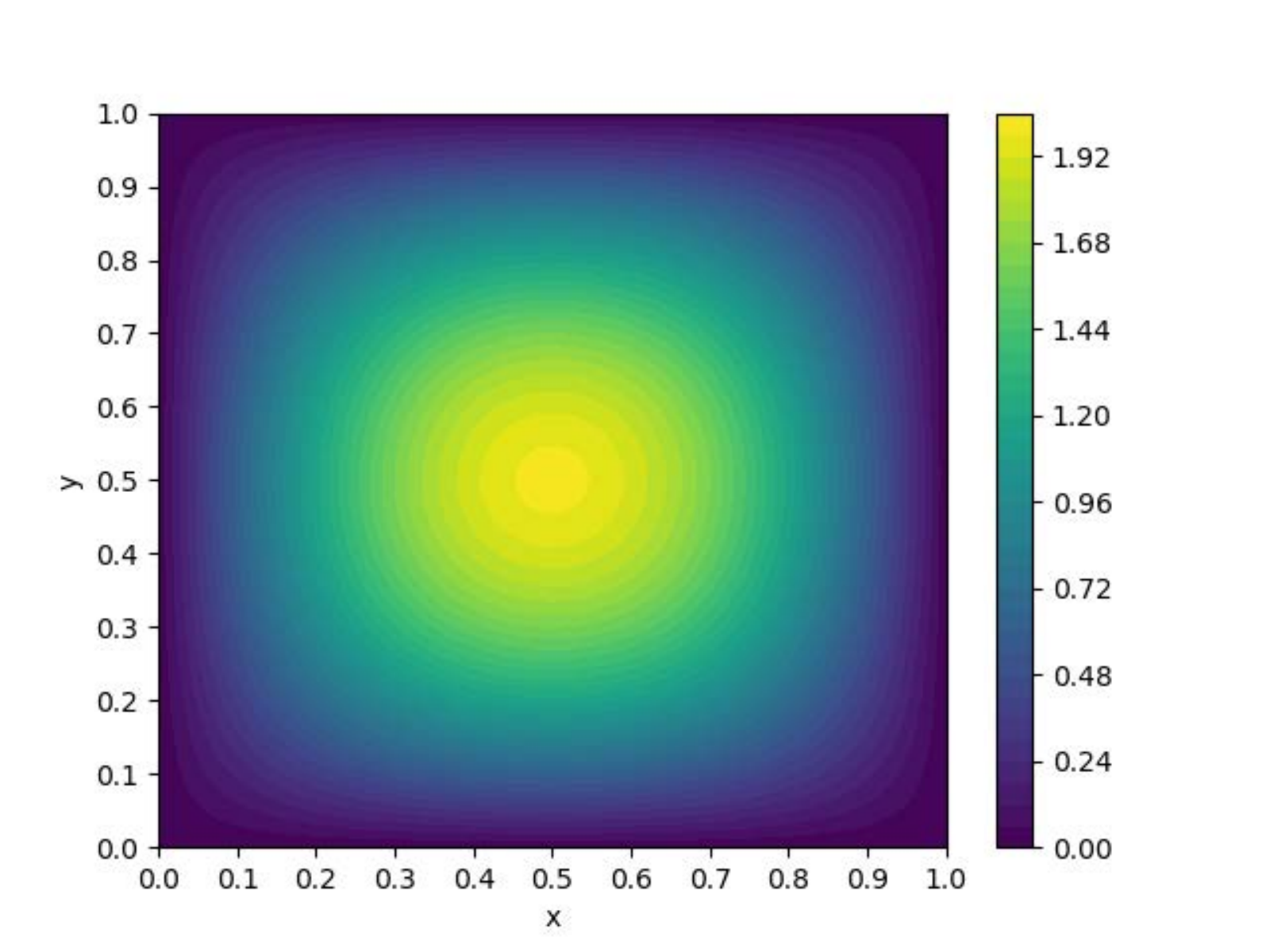}}
	\subfigure[Exact solution of $\Phi^2(r)$]{
		\centering
		\includegraphics[width=0.4\linewidth]{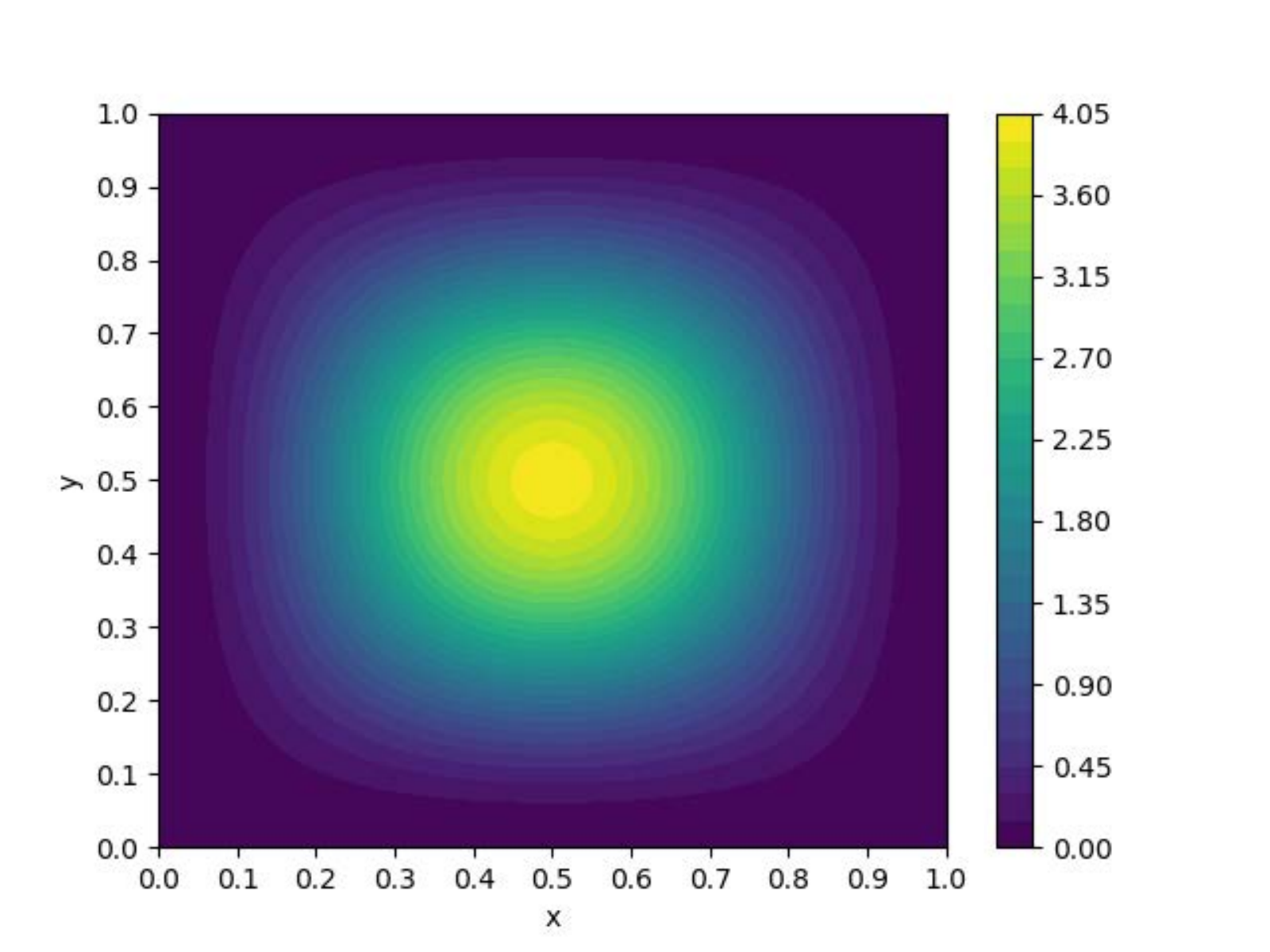}}
	\subfigure[D3M solution of $\Phi(r)$]{
		\centering
		\includegraphics[width=0.4\linewidth]{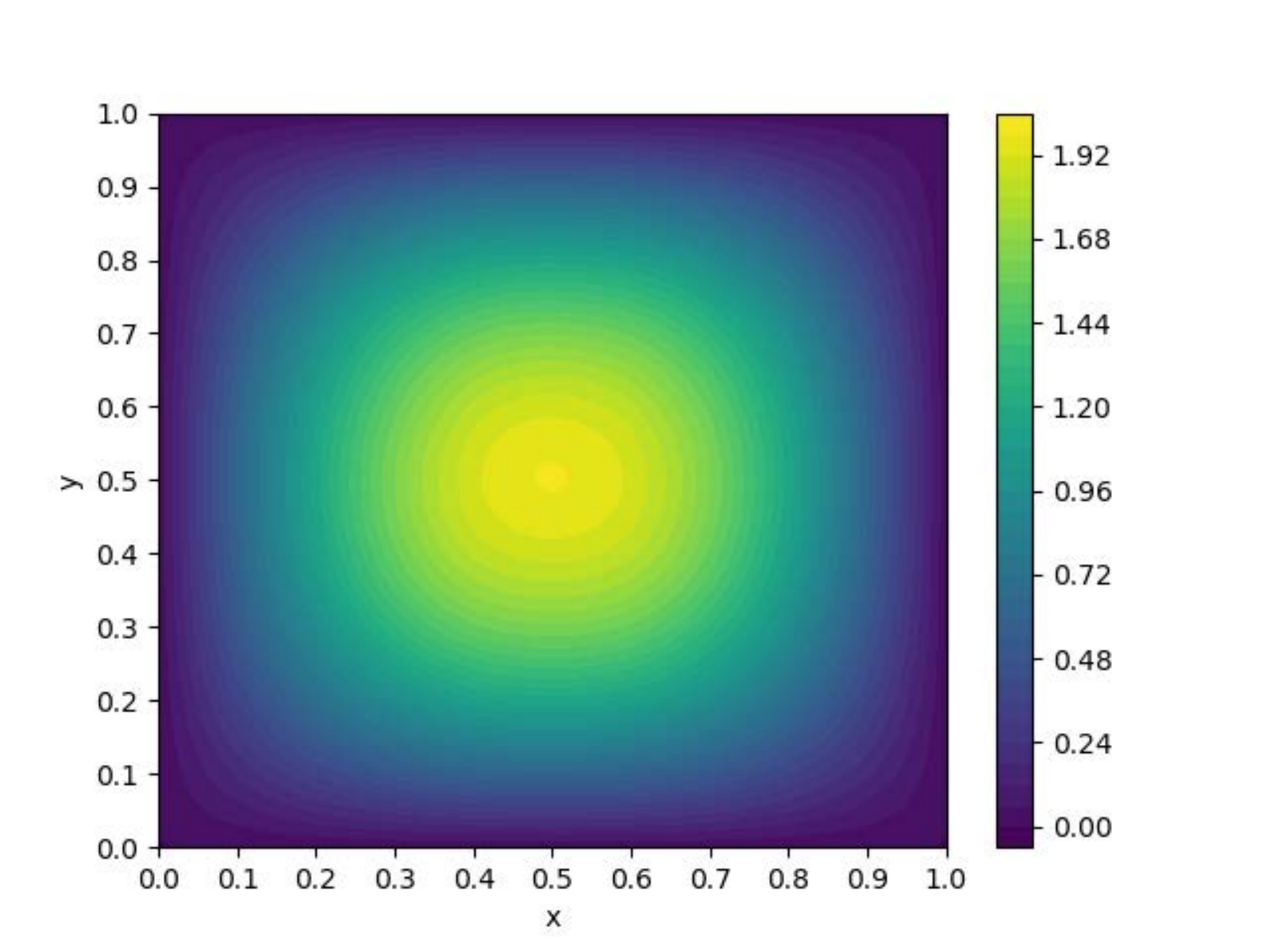}}
	\subfigure[D3M solution of $\Phi^2(r)$]{
		\centering
		\includegraphics[width=0.4\linewidth]{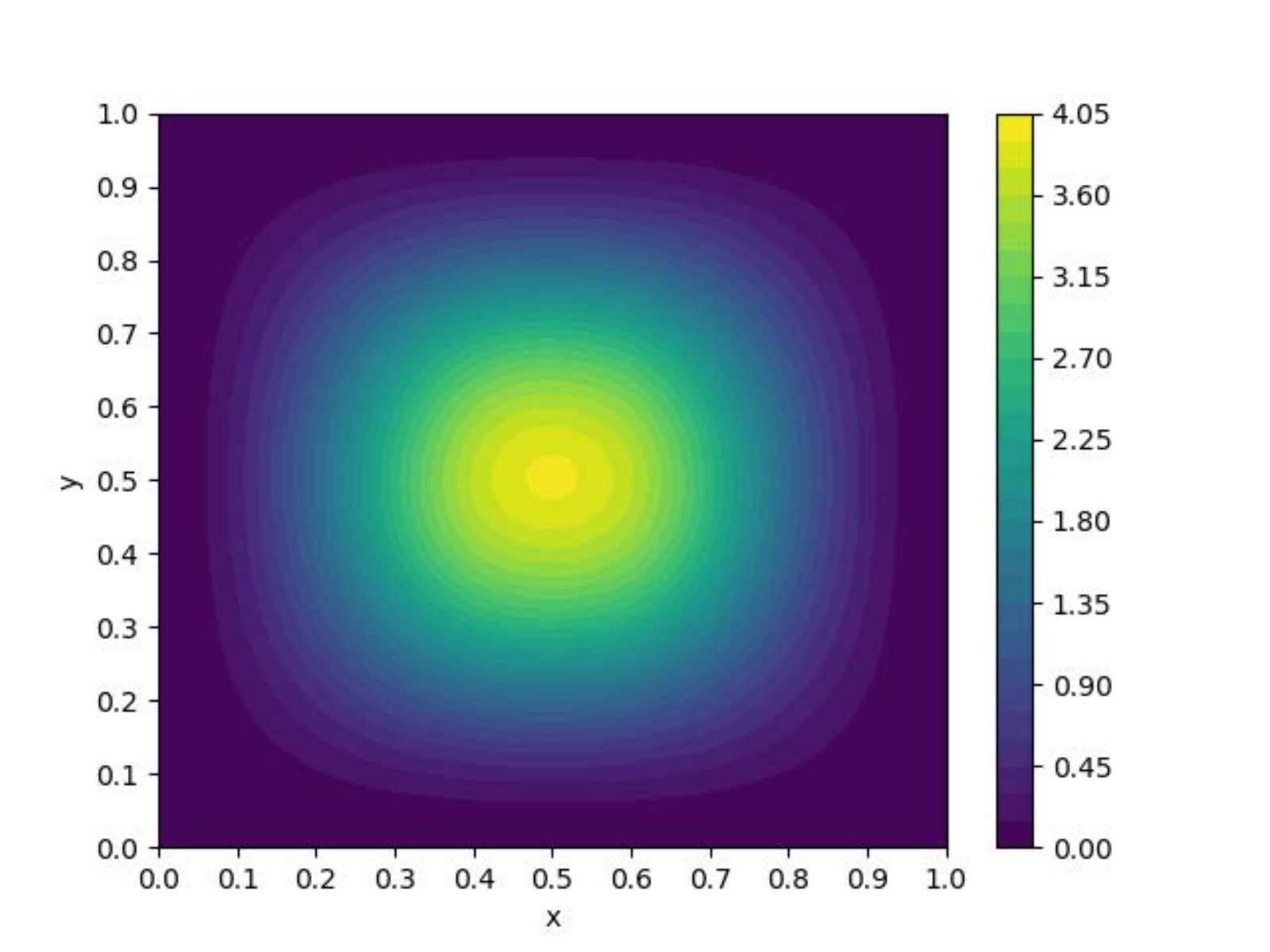}}
	\centering
	\subfigure[DRM solution of $\Phi(r)$]{
		\centering
		\includegraphics[width=0.4\linewidth]{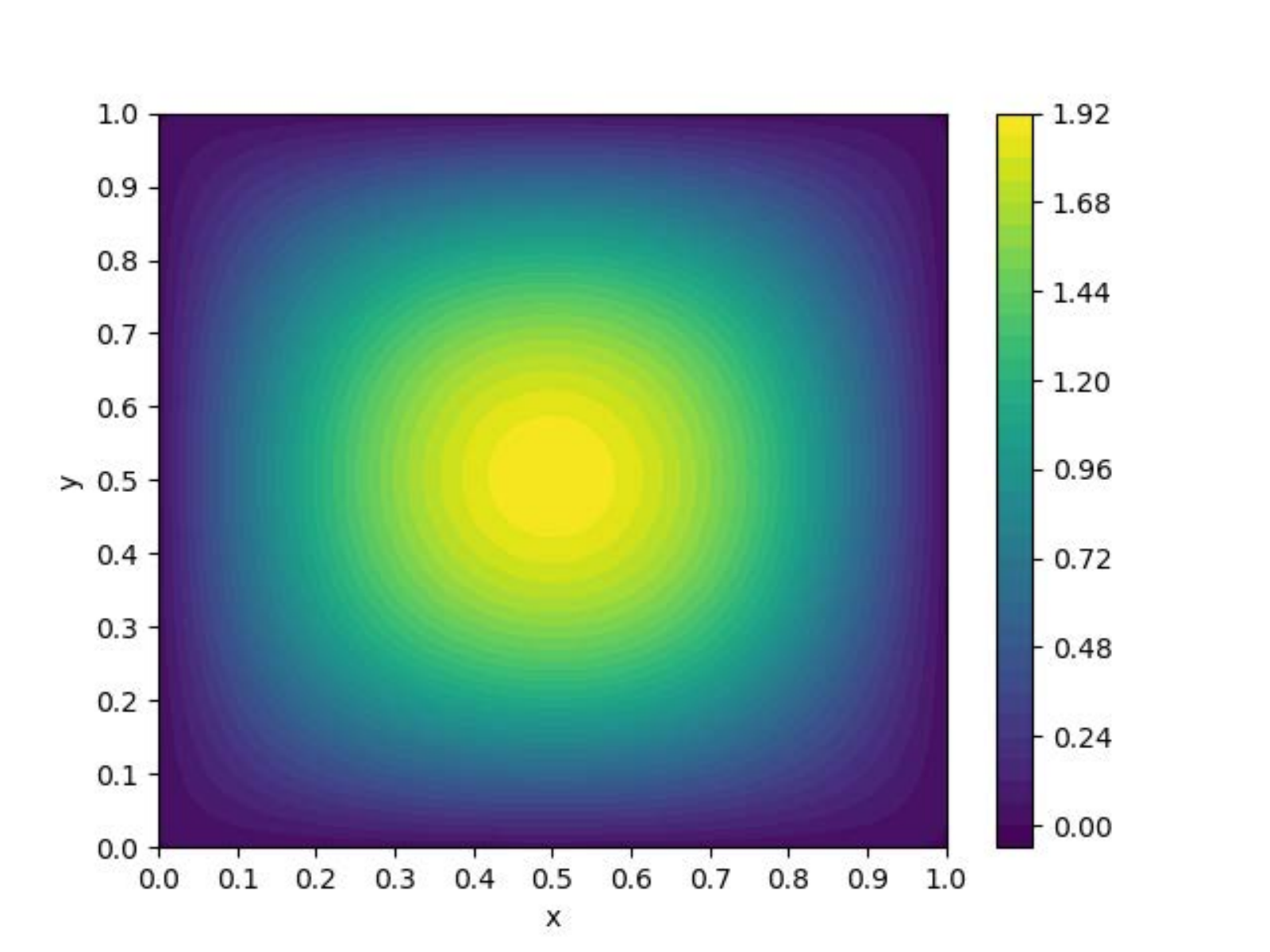}}
	\subfigure[DRM solution of $\Phi^2(r)$]{
		\centering
		\includegraphics[width=0.4\linewidth]{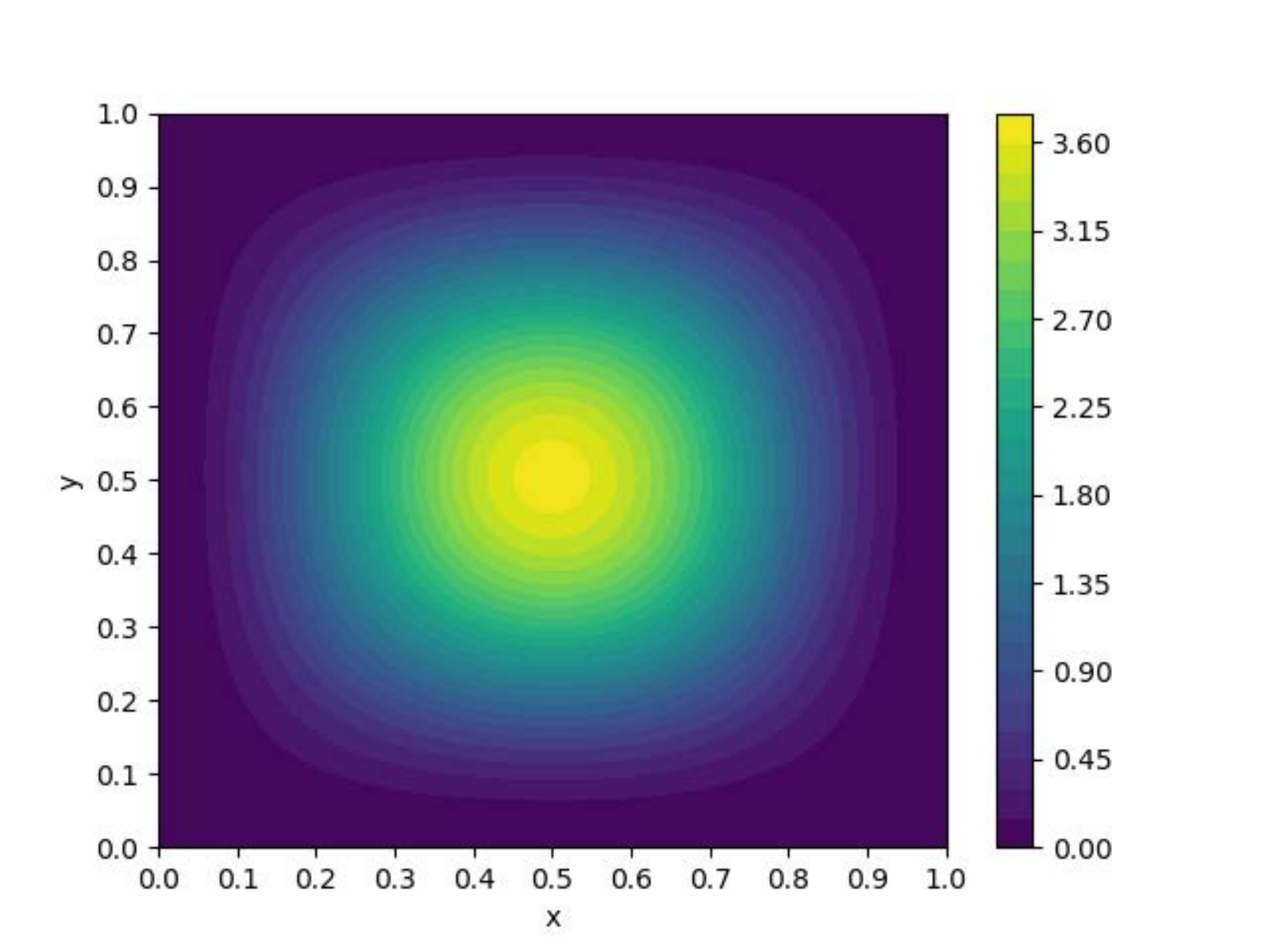}}
	\caption{The solutions of wave function and probability density for a time-independent Schrödinger equation.}
	\label{scho_fig}
\end{figure}

\begin{figure}[H]
	\centering
	\subfigure[]{
		\includegraphics[width=0.47\textwidth]{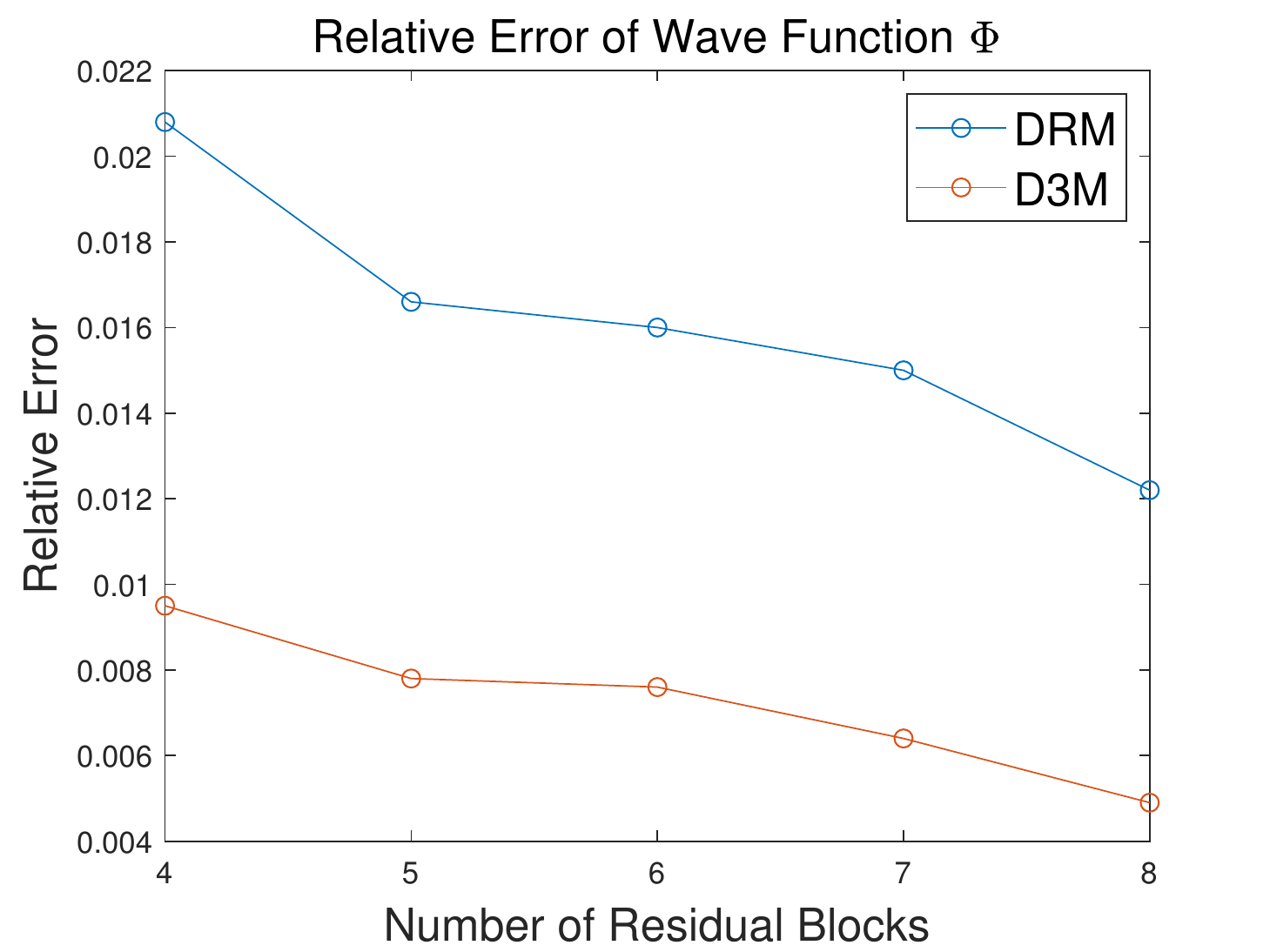}
	}
	\quad
	\subfigure[]{
		\includegraphics[width=0.47\textwidth]{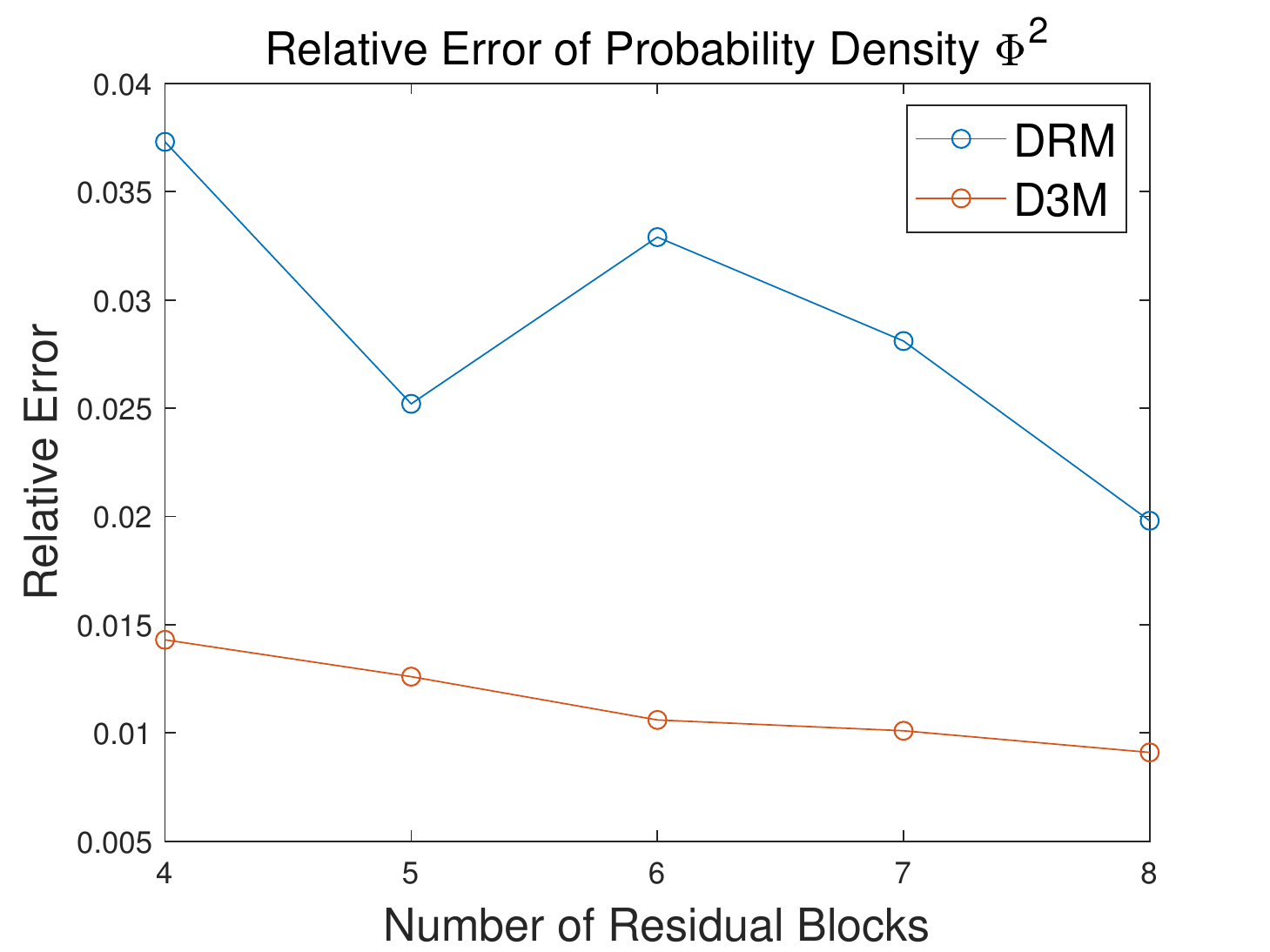}
	}
	\caption{Comparison of relative errors {\ls corresponding to} different number of residual blocks.}
	\label{schro_comp}
\end{figure}

\begin{table}[H]
	\centering
	\begin{threeparttable}
		\caption{Relative errors for the wave function and the probability density.}
		\label{table2}
		\begin{tabular}{cccccc}
			\hline
			Target &Method &Net type &Blocks &Number of neurons &Relative error\\
			\hline
			Wave&DRM&ResNet&4&2048&0.0209\\
			Wave&DRM&ResNet&8&4096&0.0169\\
			Wave&D3M&ResNet&4&2048&0.0095\\
			Wave&D3M&ResNet&8&4096&0.0045\\
			Prob.&DRM&ResNet&4&2048&0.0357\\
			Prob.&DRM&ResNet&8&4096&0.0334\\
			Prob.&D3M&ResNet&4&2048&0.0143\\
			Prob.&D3M&ResNet&8&4096&0.0091\\
			\hline
		\end{tabular}
	\end{threeparttable}
\end{table}

\section{Conclusion}\label{conclusion}
This paper has proposed a new deep domain decomposition method. The most significant contribution of the proposed approach is parallel computation, which lays a foundation for employing physics-constrained deep learning framework in large-scalar engineering simulations or designs. This is accomplished by incorporating domain decomposition method into the loss function. Based on the property of mesh-free, we propose a new D3M sampling method to improve computational efficiency. And our framework absorbs the idea of mixed finite element method, so that the boundary condition can be satisfied more accurately. We have demonstrated that deep domain decomposition method can solve general parabolic PDEs with high accuracy. Furthermore, the generalization performance of D3M should be better, because domains are smaller.

In theory, our approach is feasible to solve complex systems with many subdomains and corresponding neural networks. In practice, the approach suffers from three main bottlenecks. One is that a bad initialization of neural networks can lead to superfluous cost for following iterations. One is the choice of function for approximating interfaces. And another is that the choice of Lagrangian multiplier is important but lacks of prior. We leave these questions for future work.

\section*{Acknowledgments}
This work is supported by the National Natural Science Foundation of China (No. 11601329).


\addcontentsline{toc}{chapter}{Bibliography}
\bibliography{KeLi_d3m}

\end{document}